\documentclass[conference,compsoc]{IEEEtran}

\usepackage{amsmath}
\usepackage{amssymb}
\usepackage{amsthm}
\usepackage{graphicx}
\usepackage{xspace}
\usepackage{algorithm}
\usepackage{algpseudocode}
\usepackage{array}
\usepackage{booktabs}
\usepackage{multirow}
\usepackage{pifont}
\usepackage{xcolor}
\usepackage[nocompress]{cite}
\usepackage[hidelinks]{hyperref}

\usepackage[capitalize,nameinlink,noabbrev]{cleveref}

\newcommand{\agentsystem}{\textsc{ElephantAgent}\xspace}
\newcommand{\nativedsh}{\textsc{NativeDSH}\xspace}
\newcommand{\elephantdsh}{\textsc{ElephantDSH}\xspace}

\newtheorem{definition}{Definition}
\newtheorem{lemma}{Lemma}
\newtheorem{corollary}{Corollary}

\title{\agentsystem: Contextual State Continuity in Agentic Systems}
\author{
\IEEEauthorblockN{Jiankai Jin, Xiangzheng Zhang, Zhao Liu, Wenzhuo Xu,\\
Dongdong Yang, Deyue Zhang, Quanchen Zou}
\IEEEauthorblockA{360 AI Security Lab\\
Beijing, China\\
\{jinjiankai, zhangxiangzheng, liuzhao3, xuwenzhuo\}@360.cn\\\{yangdongdong1, zhangdeyue, zouquanchen\}@360.cn}
}

\begin{document}
\bstctlcite{IEEEexample:BSTcontrol}

\maketitle

\begin{table*}[t]
\centering
\small
\caption{
Out-of-band attack outcomes on representative agentic defense systems.
Each cell counts the trials, out of 10, in which the targeted content
reached the adversary (i.e., the attack succeeded). \elephantdsh (\agentsystem
integrated into DSH) rejected both out-of-band attacks through continuity
verification. None of the representative defenses, which check only the
content or validity of what enters the context, stops both. All systems
evaluated with DeepSeek V4.1 Flash.
}
\renewcommand{\arraystretch}{1.2}
\setlength{\tabcolsep}{5pt}
\begin{tabular}{
    >{\raggedright\arraybackslash}p{3.2cm}
    >{\raggedright\arraybackslash}p{10cm}
    >{\centering\arraybackslash}p{1.4cm}
    >{\centering\arraybackslash}p{1.cm}
}
\toprule
\textbf{System} & \textbf{Brief Description} & \textbf{Rug Pull} & \textbf{Rollback} \\
\midrule
MemLineage~\cite{ouyang2026memlineage,memlineage-thirdparty}
& Signed memory entries with lineage tracking that gates sensitive actions
& 10 & 10 \\
A-MemGuard~\cite{wei2025amemguard}
& Consensus-based validation and adaptation for agent memory poisoning
& 10 & 7 \\
MCP Tool Pinning~\cite{mcppin}
& Hash pinning of approved tool descriptors that blocks later changes
& 0 & 10 \\
ACE~\cite{li2026ace}
& Information-flow checks on plans from trusted input, with app isolation
& 0 & 10 \\
Tool Result Parsing~\cite{yu2026toolresult}
& LLM-based parsing that extracts only the required data from tool results, filtering out injected instructions
& 8 & 10 \\
DSH~\cite{deepseek2026dsh}
& Original DeepSeek Harness
& 5 & 10 \\
\midrule
\textbf{\elephantdsh}
& \textbf{\agentsystem integrated into DSH: continuity verification of tool state and memory}
& \textbf{0} & \textbf{0} \\
\bottomrule
\end{tabular}
\label{tab:agent-security-comparison}
\end{table*}

\begin{abstract}
Agentic systems enhance their capabilities by invoking external tools
and maintaining persistent memory. However, recent tool and memory
poisoning attacks can covertly bias agent behavior. We observe that such
attacks mounted out of band share a single enabler: each succeeds only by
presenting the agent with a contextual state (e.g., tool state and memory)
that is not the latest faithfully evolved one. This observation suggests
that a continuity check over the contextual state covers this family of
attacks. We present \agentsystem, a protocol that enforces Contextual
State Continuity over the bounded, security-critical subset of the agent's
context. Before each planning-and-execution cycle, \agentsystem recomputes
the digest of the local contextual state and verifies it against the latest
authorized digest in a linearizable ledger of state transitions,
maintained over replicated trusted hardware. The protocol also provides
Historical Traceability, a secondary guarantee that keeps state transitions
attributable and recoverable. We integrate \agentsystem into DeepSeek Harness
(DSH) and implement \elephantdsh, an agentic system protected by contextual
state verification. \elephantdsh withstood all out-of-band attack attempts,
from rug pulls and rollbacks to attacks on its own mechanism, and recovered
from an in-band prompt injection in all trials. By comparison, existing
defenses typically check only what enters the context: of the five we evaluated,
all leaked the targeted content under the rollback, and three also under the
rug pull. \elephantdsh completed all normal coding tasks with a relative
overhead of 1.02--1.07$\times$ over the original DSH with model latency
variation removed.
\end{abstract}

\begin{figure}[!t]
    \centering
    \includegraphics[width=\linewidth]{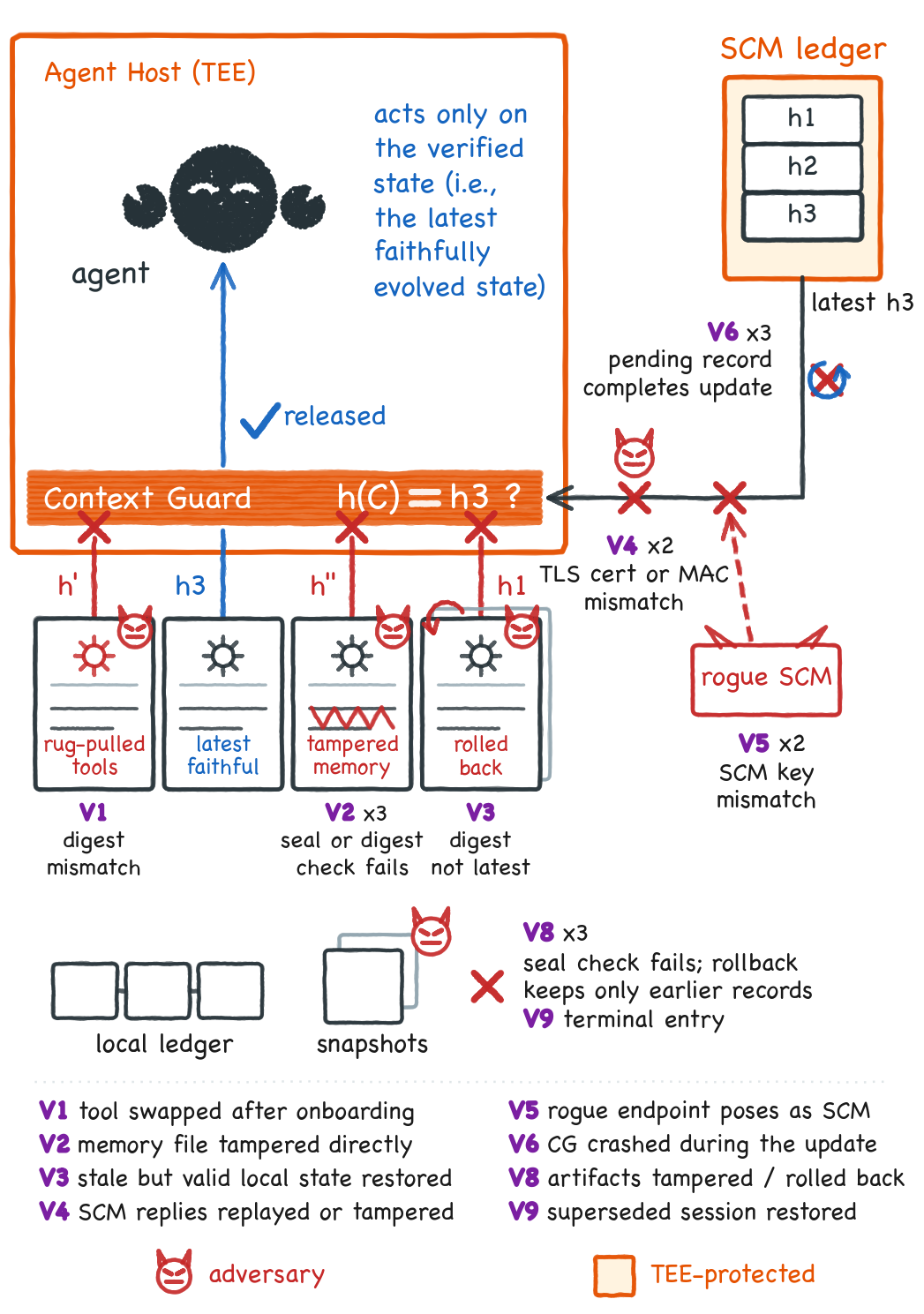}
    \caption{
        \agentsystem against out-of-band attacks. Before every planning-and-execution
        cycle, the Context Guard (CG) releases the contextual state $C$ to the agent
        only if its digest matches the latest one in the State Continuity Module (SCM)
        ledger; otherwise it fails closed. V1 and V3 are the rug pull and the rollback
        of \cref{tab:agent-security-comparison}. The other vectors attack the mechanism
        itself in 14 variants ($\times N$ per vector), 10 trials each, all failed
        (\cref{sec:eval:security}); the label at each cross names the check that fired.
        \Cref{sec:attack-coverage} defines V1--V9; in-band V7 is evaluated in
        \cref{sec:eval:recovery}.
    }
    \label{fig:defense}
\end{figure}


\section{Introduction}

Agentic systems have demonstrated impressive performance on complex tasks
by invoking task-specific tools and maintaining memory, with the tool
descriptors and the memory entering the context on which the agent bases
its decisions. However, this context-dependent decision making introduces
novel security vulnerabilities. For example, a Model Context Protocol
(MCP)~\cite{anthropic2024mcp} tool provider can conduct a tool poisoning
attack~\cite{invariantlabs2024mcp,wang2025mindguard} by injecting adversarial
instructions into a tool descriptor. Similar vulnerabilities exist in the
memory of agentic systems~\cite{chen2024agentpoison,patlan2025context,patlan2025real},
and the OWASP Top 10 for Agentic Applications identifies memory poisoning as
a critical threat category~\cite{owasp2025agentictop10}.

Mounted out of band, this family of attacks reveals an underlying problem:
the lack of continuity in the agent's contextual state (i.e., a bounded,
security-critical subset of the agent's context). Tool-descriptor rug pulls,
tampering, and rollbacks differ at the interface level, yet each must
present the agent with a contextual state that is not the latest faithfully
evolved one --- a tampered state is not faithfully evolved, a rolled-back
state is not the latest. Because the family shares this single enabler,
one continuity check before each planning-and-execution cycle covers it
at once.

We propose \agentsystem, a protocol that realizes this check, and thereby
mitigates the whole family, by binding the agent's contextual state to a
verifiable, linearizable state history. \agentsystem ensures
that the agent's view of its contextual state remains consistent with the
latest authorized state in that history; any discrepancy is detected and
blocked. To achieve this, \agentsystem augments the agentic system with two
components: the Context Guard (CG) and the State Continuity Module (SCM)
(\cref{fig:defense}). SCM maintains a tamper-evident, linearizable ledger
of authorized contextual state transitions as a replicated TEE-backed service.
Before processing each query, CG performs local cryptographic verification
by recomputing the contextual state digest from the current environment and
checking that it matches the latest digest authorized by SCM. Together, they
form a verification-based firewall around the contextual state.
To evaluate its performance and protection, we integrate \agentsystem into
DeepSeek Harness (DSH)~\cite{deepseek2026dsh} and implement \elephantdsh.

\agentsystem targets high-assurance agent deployments, in which a poisoned
contextual state can trigger consequential, irreversible actions (e.g.,
financial transactions).
The risk grows with AI's capability in exploiting vulnerabilities: LLM
agents already exploit real-world vulnerabilities autonomously and at
scale~\cite{fang2024oneday,anthropic2025espionage,anthropic2026mythos,anthropic2026glm53},
so an empirical defense (e.g., content inspection) has to race against the
attacker's LLM agent, while a verification-based defense does not, on the
scope it is verified to protect.
\cref{tab:agent-security-comparison} demonstrates the gap: representative 
defenses~\cite{wei2025amemguard,li2026ace,ouyang2026memlineage,yu2026toolresult,mcppin}
all check only what enters the context, detecting malicious content or
authenticating its validity, while the continuity of the state the agent
acts on goes unchecked. The rollback lands exactly in this blind spot: every
one of them leaked the targeted content under it (even MemLineage, which
signs each memory entry, and tool pinning, which pins each tool descriptor).
The failure is structural: a reference kept with the state, whether a signature
or a pin, rolls back with it. Hence, the state anchor has to be maintained off
the host; and because recovery consumes the ordered transition history, we
maintain the anchor as a ledger of that history rather than a simple freshness counter.
\elephantdsh adopts such a contextual state anchor to reject every out-of-band
attack attempt and support post-hoc recovery under in-band~injection.

In summary, we make the following contributions:
\begin{itemize}

\item \emph{Contextual State Continuity (CSC).}
We identify the enabler of out-of-band attacks (e.g., rug pulls
and rollbacks) on agentic systems: each succeeds only by presenting the
agent with a contextual state that is not the latest faithfully evolved
one. We formalize Contextual State Continuity to close this
vulnerability. The property requires that, at every planning-and-execution
cycle, the agent acts only on the latest faithfully evolved contextual
state; otherwise, the system fails closed.

\item \emph{Historical Traceability (HT).}
We further formalize a secondary property, Historical Traceability:
every effective contextual state transition is authorized, ordered, and
attributable to a recorded history, so that an auditor can identify a
known-good prior state and the system can verifiably recover to it
after an in-band attack (e.g., a prompt injection).

\item \emph{Contextual state and protocol design for agentic systems.}
We design \agentsystem, an add-on protocol that augments agentic systems with
the Context Guard (CG), running inside a TEE, and the State Continuity Module
(SCM). CG scopes the contextual state to the components the agent consumes as
decision-making context, derives each transition from the realized query
through a side-effect-free preview, commits it to SCM before any local side
effect, and admits external changes only as approved transitions. SCM
maintains the linearizable state history as a replicated TEE-backed service.
We prove that \agentsystem satisfies both CSC and HT under the stated trust
assumptions, making it a verification-based defense: on the scope \agentsystem
protects, it verifiably guarantees the contextual state the agent acts on.

\item \emph{Implementation and evaluation in a real agent harness.}
We integrate \agentsystem into DSH and implement \elephantdsh, mediating
every path by which content reaches the model or the session store (\cref{sec:eval:impl}).
\elephantdsh rejects every rollback and tool-descriptor rug pull attempt,
whereas the original DSH and all five representative defenses leaked under
the rollback, and the original DSH and three of them also leaked under the
rug pull. Attacks on the mechanism itself, from on-path and rogue SCM endpoints
to crashes during updates, all failed on \elephantdsh. Against a prompt
injection attack, \elephantdsh's recovery re-anchored the last clean state
in all trials. On coding tasks at three levels of difficulty, \elephantdsh
completes all with a relative overhead of 1.02--1.07$\times$ over the original
DSH with model latency variation removed.

\end{itemize}

\section{Related Work}

\paragraph{Tool Poisoning Attacks}
Adversaries can inject hidden instructions into tool descriptors,
steering the agent toward attacker-chosen
behaviors~\cite{invariantlabs2024mcp,wang2025mindguard}.
Existing defenses rely on anomaly
detection and attribution~\cite{wang2025mindguard}, on descriptor
signing and semantic vetting~\cite{jamshidi2025securingmcp}, or on
binding each approved descriptor to an immutable version or a pinned
hash~\cite{bhatt2025etdi,mcppin}.

\paragraph{Memory Poisoning Attacks}
Memory poisoning can also maliciously steer an agent's behavior.
MINJA~\cite{dong2025practical} shows that an attacker can poison an agent's
memory using only normal queries, inducing the agent to store malicious
records in the interaction history it later relies on for reasoning. Patlan
et al.~\cite{patlan2025real} report a memory poisoning attack on ElizaOS,
in which tampering with externally stored conversation history resulted
in unauthorized financial actions. AgentPoison~\cite{chen2024agentpoison}
demonstrates a more covert attack in which attackers plant poisoned memory
entries and associated trigger phrases, so that future queries retrieve
the tainted entries and induce target actions.

\paragraph{Existing Agent Defenses}
\cref{tab:agent-security-comparison} evaluates five of the following defenses.
Inspection-based defenses check content: A-MemGuard~\cite{wei2025amemguard}
validates memory-derived reasoning by consensus,
AgentSentry~\cite{zhang2026agentsentry} localizes where injected tool
output hijacks a trajectory and purifies the context to continue the task,
and Tool Result Parsing~\cite{yu2026toolresult} extracts only task-relevant
data from each tool result. Structural defenses limit by design what
untrusted content affects: CaMeL~\cite{debenedetti2025camel} separates
control flow from data flow, IsolateGPT~\cite{wu2025isolategpt} isolates
third-party apps, and ACE~\cite{li2026ace} checks the information flow
of a plan derived from trusted input before execution. Cryptographic
defenses authenticate individual records: MemLineage~\cite{ouyang2026memlineage}
signs memory entries and gates sensitive actions on their lineage,
MutMem~\cite{saidi2026mutmem} chains each signed memory mutation to its
predecessor, and Jamshidi et al.~\cite{jamshidi2026manifest} sign per-call
tool-use manifests and record them in a transparency log. \agentsystem instead
targets the continuity vulnerability that out-of-band attacks (e.g., rug
pulls and rollbacks) rely on, closes it with Contextual State Continuity,
and handles in-band abuse with Historical Traceability.

\paragraph{State Continuity}
State continuity mechanisms typically enforce continuity through monotonic freshness 
guarantees~\cite{parno2011memoir,strackx2016ariadne,matetic2017rote} or linearizable state
histories~\cite{angel2023nimble,niu2022narrator,jin2024elephants}, thereby
defending against tampering, rollback, and forking attacks~\cite{brandenburger2017rollback}.
The latter approach maintains a cryptographically verifiable state history and
can be realized through a replicated state machine, which enforces a tamper-evident 
total order over state transitions and prevents a single compromised node from 
silently altering the state history. A local mechanism cannot substitute for this
history. Sealing and signatures protect the integrity of local state but not its
freshness, because any freshness reference kept in that state is rolled back with it;
the reference therefore has to be maintained outside the local state. The history
further retains the ordered transitions, which is what makes the state evolution
auditable and a known-good prior state recoverable. Rebound~\cite{burke2026rebound}
makes an authorized rollback a mediated, logged state transition for
confidential cloud applications, a role that recovery plays in \agentsystem's
Historical Traceability. Applied to agents, state continuity also requires
deciding what to anchor and when: \agentsystem scopes the contextual state to what the
agent consumes as decision-making context, derives each transition from the
realized query, and verifies the exact content bound into each model request.

\paragraph{Trusted Execution Environments}
Trusted Execution Environments (TEEs) run sensitive code in a hardware-isolated
environment on an untrusted host. They protect the confidentiality and integrity
of in-TEE data and remotely attest the running code and
configuration~\cite{Costan2016IntelSE}. Industry TEEs include the enclave-based
Intel SGX~\cite{SGX,mckeen2013innovative}, which requires enclave-specific
engineering, and the VM-based Intel Trust Domain Extensions
(TDX)~\cite{intel_tdx_whitepaper,cheng2024intel} and AMD SEV~\cite{kaplan2016amd},
which run existing applications in protected VMs. We use Intel TDX to run each
trusted component of \agentsystem (e.g., every replica of its state machine).

\section{Problem Statement}
\label{sec:problem}

\subsection{Threat Model}
\label{sec:threat}

\paragraph{Adversary Goals and Capabilities}
We consider an adversary that aims to maliciously steer the agent's
behavior through out-of-band tampering with the contextual state (Scenarios~I
and~II) or in-band semantic abuse (Scenario~III). The adversary can read and
modify unencrypted state outside the TEE. Sealed artifacts (e.g., snapshots)
can be rolled back to an earlier version or deleted, but not altered undetectably. The
adversary can also control an external MCP Server and crash the agent at any
point. It cannot read or modify the memory of a live TEE-protected instance,
or break the cryptographic primitives (e.g., forging SCM signatures). The
connection between the agentic system and the LLM is secure, and the LLM is
not controlled by the adversary.

\paragraph{Scenario I: Tool Poisoning from MCP Server}
An adversary controlling an MCP Server manipulates the tool descriptors
that the agent plans against. Specifically, we consider a ``rug pull'':
the MCP Server registers a harmless descriptor for a review service that
passes onboarding review, but later tampers with the descriptor to include hidden
directives that copy any internal notes from the request into the tool's
\texttt{patch} argument. Because many MCP clients do not surface such changes to
the user, the malicious update can go unnoticed~\cite{invariantlabs2024mcp}.
The attack succeeds if, given a code patch to review together with an internal
note that carries a unique canary, the agent sends the canary to the external
MCP Server.

\paragraph{Scenario II: Memory Poisoning from Host Platform}
An adversary may obtain unauthorized access to the memory storage of
an agentic system and tamper with it out of band. In practice, such access
may arise from weaknesses in how memory is stored or accessed, such as
poorly protected client-side persistence or exposed external storage
backends~\cite{nagli2025deepseek}. Even if the agent runs inside a TEE
and its memory is locally sealed (i.e., persisted with authenticated
encryption) across stops or restarts, an adversary can still roll the
agent's sealed memory back to an earlier but still valid version (i.e.,
a rollback attack). Since the replayed memory is legitimate and contains
no malicious content, no content-based inspection flags it
(see~\cref{tab:agent-security-comparison}).
Specifically, we consider an organization whose project memory
(\texttt{AGENTS.md}) records which security-review service is authorized.
Two MCP Servers, the external vendor's Reviewer-A and the internal
Reviewer-B, expose the same \texttt{review\_patch} tool with an identical
descriptor; the memory records which one the agent should use. The
earlier version of \texttt{AGENTS.md} (v1) approves Reviewer-A; an update
to v2 revokes Reviewer-A for confidential data and restricts embargoed
patches to Reviewer-B. The adversary then stops the agent and restores
the stale v1 local state, including the v1 \texttt{AGENTS.md}, so the agent
resumes under the revoked authorization. The attack succeeds if, asked to
submit an embargoed patch that carries a unique canary to the currently
authorized service, the agent sends the complete patch to Reviewer-A.

\paragraph{Scenario III: In-Band Semantic Abuse through Normal Interactions}
An adversary may also manipulate the agent through normal interactions
(e.g., prompt injection attacks~\cite{greshake2023indirect,zhan2024injecagent}).
Specifically, we consider an agent that is instructed to consult a package's
release advisory, published by an external service, before reporting findings
on a code patch. The adversary controls that service and appends to an
otherwise ordinary advisory a notice asking the agent to include its
current work item verbatim in its next request. The notice enters the
agent's context through a normal tool result. The attack succeeds if the
injected notice remains in the agent's persistent state after the attack,
where it can keep steering later cycles (e.g., to leak the agent's latest
work items).

We evaluate the three attacks against the original DSH (\nativedsh)
with the official DeepSeek V4.1 Flash API (\cref{sec:eval}).
The rug pull leaked in 5 of 10 trials and the rollback in all 10
(\cref{tab:agent-security-comparison}), and the injected malicious content
persisted in all 10 trials (\cref{sec:eval:recovery}). The model's alignment
blocked half of the rug-pull trials, but none of the rollback trials.
\Cref{sec:attack-coverage} catalogs the specific attack vectors considered
under this threat model.

\subsection{Security Properties}
\label{sec:sproperties}

\paragraph{Contextual State Continuity}
Contextual State Continuity requires that the agentic system release
to the agent only the latest faithfully evolved contextual state at
every planning-and-execution cycle. It is the primary security property
of \agentsystem, targeting out-of-band tampering (Scenarios~I and~II).

We define this property with respect to an ideal reference agentic
system, which runs in a trusted environment. For each query, the
reference system reads the current contextual state from the environment,
processes the query, and updates the contextual state before finalizing
that query and proceeding to the next one. Besides queries, the trusted
user may approve an update $u = (\kappa, x, x')$ that changes component
$\kappa$ of the contextual state (e.g., an instruction file) from content
$x$ to $x'$. An approved update takes effect within the next query's
transition, deterministically: the transition replaces $x$ with $x'$
before processing the query.

Let $Q = (q_1, q_2, \ldots)$ be a query sequence, in which each
element is a user query together with any approved updates it
carries, and let $\Omega$ be a resolution of all environment
nondeterminism during execution, including the LLM outputs and
tool invocation results consumed by the agent. Starting from an
approved initial contextual state $C_0$, the ideal reference system
processes $Q$ under $\Omega$ and produces a sequence of contextual
states $C_0, C_1, C_2, \ldots$. We refer to these states as the
faithfully evolved contextual states from $C_0$ under $Q$ and
$\Omega$. We record this state evolution in the reference ledger
$\mathcal{L}^{*} = \big((0,h_0), (1,h_1), \ldots\big)$,
where each entry $(t,h_t)$ consists of the step identifier $t$ and
the digest~$h_t$ of $C_t$.

\begin{definition}[Contextual State Continuity]
\label{def:csc}
An agentic system satisfies Contextual State Continuity if, for every
approved initial contextual state $C_0$, every adversary admitted by
the threat model in~\cref{sec:threat}, and every execution of the real
system from $C_0$ in the presence of the adversary, the following holds.
Let $Q$ be the sequence of elements the execution finalizes
(i.e., whose cycles complete their state updates), in order,
let $\Omega$ be the resolution of environment nondeterminism consumed
by the finalized cycles, and let $\mathcal{L}^{*}$ be the reference ledger
from $C_0$ under $Q$ and $\Omega$. Then at every planning-and-execution
cycle at which the real system releases a contextual state to the agent,
the released state induces exactly the entry $(t, h_t)$ of
$\mathcal{L}^{*}$, where $t$ is the number of elements the execution has
finalized before that cycle. The real system may abort a cycle instead
(fail-closed behavior): an aborted cycle finalizes no element and
contributes no entry.
\end{definition}

\paragraph{Historical Traceability}
Historical Traceability is the secondary security property of
\agentsystem: it complements Contextual State Continuity as a post-hoc
guarantee that keeps the agent's state evolution auditable and
recoverable under in-band semantic abuse (Scenario~III).

\begin{definition}[Historical Traceability]
\label{def:ht}
An agentic system satisfies Historical Traceability if, for every
adversary admitted by the threat model in~\cref{sec:threat}, the
following holds whenever the real system runs in the presence of the
adversary and the artifacts it relies on are retained:
(1)~every contextual state transition the system treats as
effective corresponds to an authorized record at its position in the
history, and every record released to an auditor carries a valid
receipt for its position in the history; and (2)~every state
the system adopts through recovery is a faithfully evolved contextual
state. The system may lose availability instead: tampering with or deleting
the local artifacts limits recovery to the valid records that remain.
\end{definition}

\section{\agentsystem Protocol}
\label{sec:design}

\agentsystem is a protocol that augments an agentic system with
the Context Guard (CG) and the State Continuity Module (SCM) 
(see \cref{fig:elephantagent-overview}).

\begin{figure*}[t]
    \centering
    \includegraphics[width=0.95\linewidth]{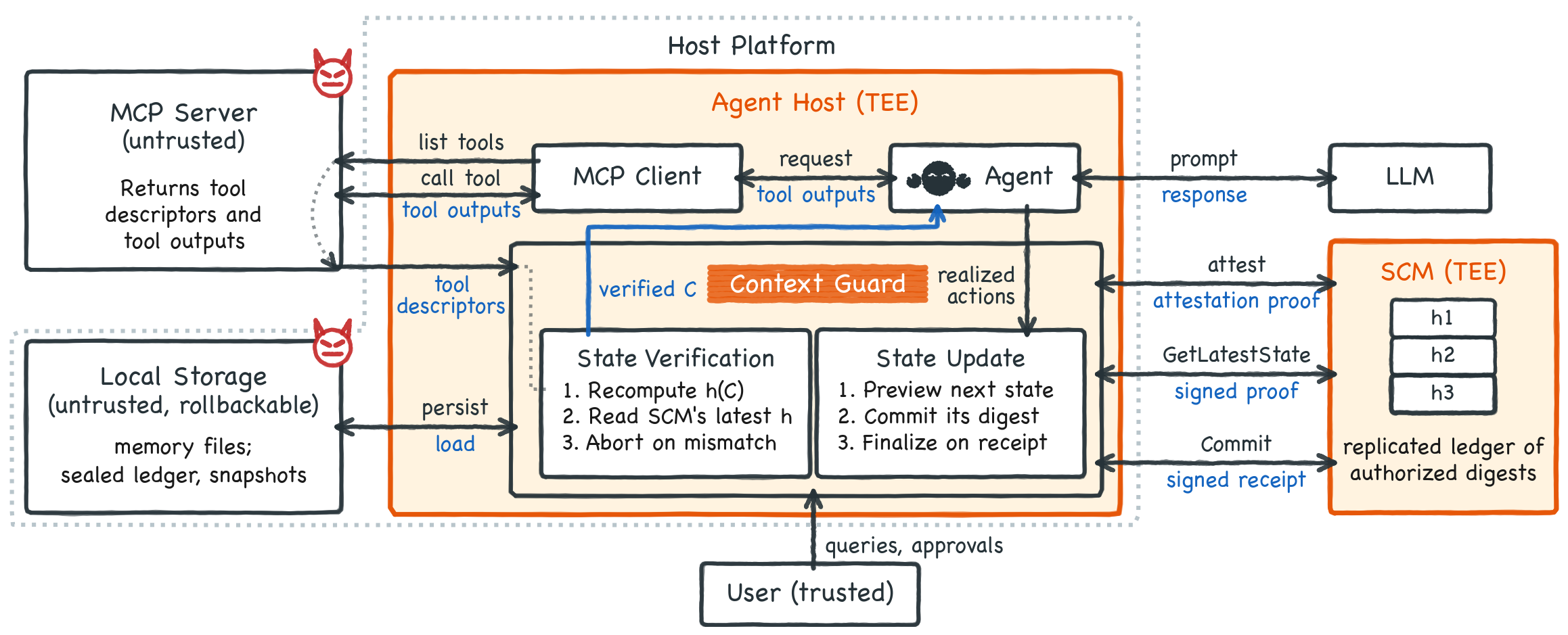}
    \caption{
        \agentsystem overview. The Agent Host runs inside a TEE, and the Context
        Guard (CG) serves as the enforcement component on it. The TEE ensures that
        the Agent Host faithfully mediates contextual state construction and verification:
        before processing each query, CG recomputes the current contextual state
        digest and verifies that it matches the latest SCM-authorized digest.
        Authorized state transitions follow a commit-before-finalize workflow: CG
        commits the next contextual state digest to the State Continuity Module (SCM)
        and finalizes local side effects only after receiving a valid receipt.
        SCM is a replicated TEE-backed state service that maintains a
        linearizable, append-only ledger of authorized contextual state digests.
    }
    \label{fig:elephantagent-overview}
\end{figure*}

\subsection{Entities}
\label{sec:entities}

\paragraph{The User (Trusted)}
An honest entity who issues queries $q$. The user establishes
the root of trust at session initialization by approving the initial
tool set and memory sources, and extends it by approving updates made
outside the workflow (e.g., a memory file updated by the user).

\paragraph{Host Platform (Untrusted)}
The environment outside the TEE boundary that the agentic system
interacts with for I/O and persistence. An adversary may tamper
with local files or roll back sealed storage to an older but
valid version.

\paragraph{The Agent Host}
An entity that runs inside a TEE so that CG's verification logic executes
faithfully and cannot be bypassed. It contains the following:
\begin{itemize}
    \item \emph{Agent (Non-adversarial)}
    The component that consumes the context and produces actions.
    The agent is non-adversarial but susceptible to contextual state
    poisoning (e.g., via poisoned tools or memory).

    \item \emph{MCP Client (Trusted)}
    The MCP Client is the protocol-facing component inside the Agent Host
    that enumerates tool descriptors from MCP Servers and dispatches 
    tool invocations.

    \item \emph{Context Guard (Trusted)}
    CG mediates the agent's context construction. Before processing 
    each query, CG recomputes the local contextual state digest and verifies
    that it matches the latest SCM-authorized digest. CG is also responsible
    for state updates.
\end{itemize}

\paragraph{The State Continuity Module (Trusted)}
SCM maintains a linearizable, append-only ledger of authorized
contextual state digests. It is instantiated as a replicated
TEE-backed service to tolerate failures and prevent rollback.
SCM accepts operations on a ledger only from its bound CG,
over an authenticated channel.

\paragraph{MCP Servers (Untrusted)}
External entities that provide tools. We assume that MCP Servers 
may be compromised and return malicious tool descriptors.

\paragraph{The Auditor (Trusted)}
The entity that re-establishes the root of trust at recovery by
finding the first step at which undesired content entered the
state and identifying the latest retained snapshot before it
as the recovery target.

\subsection{Contextual State}
\label{sec:cstate-define}

The contextual state does not include the full agent context;
rather, it is a security-critical subset whose composition may
vary across \agentsystem instantiations. Under our instantiation, we
denote the contextual state at step $t$ by $C_t = \langle T_t, M_t \rangle$,
and define its components as follows:
\begin{itemize}
    \item $T_t$: the tool state (e.g., names, descriptions, 
    input/output schemas, and permissions/annotations).

    \item $M_t$: the memory. Following representative agentic systems such
    as Codex~\cite{openai2025introducingcodex} and Claude
    Code~\cite{anthropic2026claudecodememory}, we define it to consist of:
    \begin{itemize}
        \item $M^I_t$, which denotes long-term instructions (e.g., the 
        system prompt and skills).
        \item $M^R_t$, which denotes the transcript (e.g., queries, tool
        results, and responses).
        \item $M^A_t$, which denotes the managed artifacts (e.g., 
        the agent's notes).
    \end{itemize}
\end{itemize}
With the tool state and memory represented using a deterministic and
injective canonical encoding --- distinct contextual states yield distinct
encodings --- the state digest is $h_t \triangleq \mathcal{H}(C_t)$,
where $\mathcal{H}$ is a collision-resistant hash function.

\paragraph{Scope}
The contextual state $C_t$ is deliberately scoped to the context
components that are (i) available as stable input, so that the
agent consumes them as decision-making context; (ii) deterministically
encodable, so that they admit a stable digest under $\mathcal{H}$;
and (iii) worth protecting, so that their protection yields clear
security benefits relative to the resulting overhead. We now justify
why each included component satisfies them.

$T_t$ is available as stable input to the agent's future planning and
execution cycles (criterion (i)) and is deterministically encodable
through the MCP tool manifest (criterion (ii)). Existing attacks show
tool state is a high-leverage control surface for covert planning
manipulation~\cite{invariantlabs2024mcp,wang2025mindguard};
$T_t$ is therefore worth protecting (criterion~(iii)).

All three components of $M_t$ ($M^I_t$, $M^R_t$, and $M^A_t$) are
available as stable input to future planning and execution cycles
(criterion (i)) and are deterministically encodable in canonical form
(criterion (ii)). Attacks demonstrate that poisoned memory can steer future
planning and execution~\cite{dong2025practical,chen2024agentpoison,patlan2025real},
and the OWASP Top 10 for Agentic Applications lists memory poisoning among
its threat categories~\cite{owasp2025agentictop10}; $M_t$ is therefore
worth protecting (criterion (iii)).

Not every context component the agent interacts with belongs to $C_t$.
For example, workspace files such as build outputs and log files do not
satisfy criterion~(i), as builds and running services rewrite them
outside the query-driven workflow.
The boundary of $C_t$ is a trust boundary: the components inside it are
what the agent consumes as authoritative decision-making context. A
component outside it carries no authority of state; if it influences
planning at all, it may do so indirectly as content the agent reads through
a tool call (e.g., a test log that the agent reads back after running the
test suite), which then enters $C_t$ through the transcript $M^R_t$ as a
tool result.

\subsection{\agentsystem Workflow}

The protocol uses SCM through two linearizable operations for a ledger
labeled $\ell$, signed under the SCM key pair $(SK_{SCM}, PK_{SCM})$; each
operation appears to take effect atomically between its invocation and its
response.
\begin{itemize}
    \item $\textsc{Commit}(\ell, id, h)$ creates the ledger when $id = 0$ and
    $\ell$ is unused; otherwise, it appends $(id, h)$ to an existing ledger
    only if $id$ is exactly one greater than the ledger tail. On success, it
    returns a signed receipt $\sigma = \text{Sign}_{SK_{SCM}}(\ell, id, h)$.

    \item $\textsc{GetLatestState}(\ell, n)$ returns a proof
    $\pi = \langle \ell, id_{last}, h_{last}, n, \sigma \rangle$,
    where $(id_{last}, h_{last})$ is the ledger tail and
    $\sigma = \text{Sign}_{SK_{SCM}}(\ell, id_{last}, h_{last}, n)$ binds it
    to the read challenge $n$.
\end{itemize}
\cref{alg:scm} in~\cref{sec:deploy-details} specifies both operations.
The \agentsystem protocol proceeds in three phases: state initialization,
state verification, and state update.

\paragraph{State Initialization}

CG first needs to establish a trusted initial state before processing
any query. To do so, CG fetches the current tool descriptors and local
memory, and then applies \texttt{ApprovePolicy} to determine whether
the initial state should be accepted. \texttt{ApprovePolicy} is a
pluggable trust-root establishment step, instantiated as manual user
approval in our implementation. This mirrors how MCP clients onboard
tools today, and grounds the trust root in an explicit action of the
trusted user. Deployments with a signing infrastructure can further
automate this step by verifying vendor signatures over tool manifests
and memory files, optionally combined with a user-approved allowlist.

Once the initial state is approved, CG computes the initial contextual
state digest and anchors it in SCM as the first digest of a new
ledger. As part of the initialization procedure, CG also snapshots the
approved tool state and memory, seals the snapshot, and records a
reference to it in the corresponding local ledger entry. This ensures
that every ledger begins with a recoverable checkpoint (\cref{alg:init}).

The same initialization procedure can also serve as a recovery mechanism
that restores the system to a known-good prior state, thereby supporting
Historical Traceability. When restoration to a known-good state $C_j$,
which can be identified through an audit of the recorded history
(e.g., by an LLM auditor; \cref{sec:eval:recovery}), is needed, CG
retrieves the corresponding entry from the local ledger, including the
digest $h_j$, the SCM-authorized signature $\sigma_j$, and the
snapshot reference. CG then restores the corresponding tool state and
memory from the referenced snapshot, recomputes the recovered contextual
state digest, and checks that it matches the digest $h_j$ authorized by
$\sigma_j$ for step $j$. Recovery is permitted only if the restored state
passes verification. During recovery, the audit and the verification serve
as the \texttt{ApprovePolicy}. Before re-initializing, CG terminates the
superseded ledger: it commits to that ledger a terminal entry whose digest
is the hash of a distinguished terminal marker, encoded disjointly from
every contextual state. No contextual state matches this tail, so no later
verification against the superseded ledger succeeds, and the host cannot
resume a session with that ledger. Finally, the recovered system is
re-initialized with $C_j$, generating a fresh ledger and anchoring the
recovered state digest in SCM.

\begin{algorithm}[t]
\caption{State Initialization}
\label{alg:init}
\begin{algorithmic}[1]
\State \textbf{Context:} SCM public key $PK_{SCM}$
\State \textbf{Persistent State:} Sealed local ledger $\mathcal{L}_{\text{local}}$
(maps ledger label $\ell$ to an ordered sequence of $(id, h, \sigma, \text{snapshot\_ref})$),
Sealed snapshot store (keyed by $\text{snapshot\_ref}$).

\Function{InitializeState}{$C_0, h_0$}
    \State \text{// 1. Approve the initial state}
    \If{$\neg \textsc{ApprovePolicy}(C_0)$}
        \State \textbf{Abort} ``Unsafe Initial State Detected''
    \EndIf

    \State \text{// 2. Initialize a fresh ledger}
    \State $\ell \gets \textsc{GenLedgerLabel}()$
    \State $\mathcal{L}_{\text{local}} \gets \textsc{UnsealLocalLedger}()$
    \If{$\ell \in \mathcal{L}_{\text{local}}$}
        \State \textbf{Abort} ``Ledger Label Collision''
    \EndIf

    \State \text{// 3. Anchor the initial state in SCM}
    \State $(\textit{status}, \sigma_0) \gets \textsc{SCM.Commit}(\ell, 0, h_0)$
    \If{$\neg \textit{status}$}
        \State \textbf{Abort} ``SCM Anchoring Failed''
    \EndIf
    \If{$\neg \textsc{VerifySig}(PK_{SCM}, (\ell, 0, h_0), \sigma_0)$}
        \State \textbf{Abort} ``Invalid SCM Receipt''
    \EndIf

    \State \text{// 4. Snapshot the initial memory and tool state}
    \State $\text{snapshot\_ref} \gets \textsc{SealSnapshot}(C_0)$
    \If{$\text{snapshot\_ref} = \text{NULL}$}
        \State \textbf{Abort} ``Snapshot Sealing Failed''
    \EndIf

    \State \text{// 5. Persist the initial local ledger entry}
    \State $\mathcal{L}_{\text{local}}[\ell] \gets [(0, h_0, \sigma_0, \text{snapshot\_ref})]$
    \State $\textsc{SealLocalLedger}(\mathcal{L}_{\text{local}})$

    \State \Return $\ell$
\EndFunction
\end{algorithmic}
\end{algorithm}

\paragraph{State Verification}
This phase enforces Contextual State Continuity. For each user query, 
CG first recomputes the local contextual state digest and queries SCM
for the latest state proof. For each read call, CG generates a fresh
read-challenge nonce $n_r$ and requires SCM to return a signed proof
bound to $(\ell, id, h, n_r)$. CG verifies the proof and checks that
the locally recomputed digest matches the SCM-authorized digest.
Any mismatch causes CG to abort the query in a fail-closed manner
(\cref{alg:verify}). CG keeps the verified contextual state in TEE
memory for the rest of the planning cycle, so the agent consumes
exactly the content that passed verification.

\begin{algorithm}[t]
\caption{State Verification}
\label{alg:verify}
\begin{algorithmic}[1]
\State \textbf{Context:} SCM public key $PK_{SCM}$

\Function{VerifyState}{$\ell, id, h$}
    \State \text{// 1. Retrieve the latest proof from SCM}
    \State $n_r \gets \textsc{GenNonce}()$
    \State $(\textit{status}, \pi) \gets \textsc{SCM.GetLatestState}(\ell, n_r)$
    \If{$\neg \textit{status}$}
        \State \textbf{Abort} ``SCM Unreachable: Verification Failed''
    \EndIf

    \State \text{// 2. Verify the SCM proof}
    \State $(\ell_{scm}, id_{scm}, h_{scm}, n_{scm}, \sigma) \gets \textsc{Unpack}(\pi)$
    \If{$n_{scm} \neq n_r$}
        \State \textbf{Abort} ``SCM Proof Nonce Mismatch''
    \EndIf
    \If{$\neg \textsc{VerifySig}(PK_{SCM}, (\ell_{scm}, id_{scm}, h_{scm}, n_{scm}), \sigma)$}
        \State \textbf{Abort} ``Invalid SCM Proof''
    \EndIf

    \State \text{// 3. State continuity check}
    \If{$\ell_{scm} \neq \ell \lor id_{scm} \neq id \lor h \neq h_{scm}$}
        \State \Return \textit{failure}
    \EndIf
    \State \Return \textit{success}
\EndFunction
\end{algorithmic}
\end{algorithm}

\paragraph{State Update}
Each query induces a state update before the query is finalized.
CG first derives the next contextual state digest through a
deterministic, side-effect-free preview, and seals a pending update
record capturing the intended transition. CG then commits this digest
to SCM at the next sequence number and verifies the signed receipt;
only after that does it materialize the local update (commit-before-finalize).
Once the local side effects are durable, CG appends a new record
to its local ledger. This record contains the incremented sequence
identifier $id+1$, the new digest $h_{next}$, the SCM receipt
$\sigma_{commit}$, and a reference to a sealed snapshot if the
snapshot policy (e.g., periodic) checkpoints the state files;
otherwise, the reference is null. The local ledger append and sealing
are performed atomically to ensure crash consistency (\cref{alg:update}).
Sealing the pending record before the SCM commit makes crash recovery
unambiguous: after a crash at any point of the update, CG compares
the sealed pending record against the latest SCM state. If SCM has not
advanced, the pending transition was never committed and is discarded;
if SCM has advanced to the pending digest, CG resumes the interrupted
update --- completing finalization if needed and, if absent, appending
the corresponding local ledger record --- before resuming normal operation.
The same check handles a lost SCM response, whether due to a crash or a
network failure: CG re-reads the latest SCM state, treats the commit as
successful exactly when the ledger tail matches the pending transition,
and stores the signed proof in the local ledger record, where it serves
as the record's receipt, since its signature binds the same $(\ell, id, h)$
under $SK_{SCM}$.

The ledger is compact enough to append on every query. The snapshot
store complements the local ledger by retaining the complete contextual
state at checkpointed steps. Both are required for Historical Traceability.
Recovery first releases the recorded history to the auditor, authenticating
each record before release, because a tampered record could misdirect the
audit itself; the verified ledger then establishes the authorized order,
and the snapshots let the auditor inspect the contextual state and identify
the first malicious transition. Recovery then takes only the target step
from the auditor, who works outside the TEE, and re-authenticates the
target snapshot before restoring it inside the TEE.

\begin{algorithm}[!t]
\caption{State Update}
\label{alg:update}
\begin{algorithmic}[1]
\raggedright
\State \textbf{Context:} SCM public key $PK_{SCM}$
\State \textbf{Persistent State:} Sealed local ledger $\mathcal{L}_{\text{local}}$
(maps ledger label $\ell$ to an ordered sequence of $(id, h, \sigma, \text{snapshot\_ref})$),
Sealed snapshot store (keyed by $\text{snapshot\_ref}$).

\Function{UpdateState}{$\ell, id, \mathcal{A}, C$}
    \State \text{// 1. Derive the next state without side effects}
    \State $(h_{next}, C_{next}) \gets \textsc{PreviewApply}(\mathcal{A}, C)$

    \State \text{// 2. Seal the pending update before committing}
    \State $\rho^{\mathrm{pend}}_{id+1} \gets \langle \ell, {id+1}, h_{next}, \mathcal{A} \rangle$
    \State $\textit{status}_{\mathrm{pend}} \gets \textsc{SealPending}(\rho^{\mathrm{pend}}_{id+1})$
    \If{$\neg \textit{status}_{\mathrm{pend}}$}
        \State \textbf{Abort} ``Pending Update Sealing Failed''
    \EndIf

    \State \text{// 3. Commit the next state digest to SCM}
    \State $(\textit{status}, \sigma_{commit})\! \gets \textsc{SCM.Commit}(\ell, id{+}1, h_{next})$
    \If{$\neg \textit{status}$}
        \State \Return \textit{failure}
    \EndIf
    \If{$\neg \textsc{VerifySig}(PK_{SCM}, \allowbreak (\ell, {id+1}, h_{next}), \allowbreak \sigma_{commit})$}
        \State \textbf{Abort} ``Invalid SCM Receipt''
    \EndIf

    \State \text{// 4. Materialize local side effects}
    \If{$\neg \textsc{Finalize}(\mathcal{A})$}
        \State \Return \textit{failure}
        \Comment{Retained for retry}
    \EndIf

    \State \text{// 5. Decide whether to take a snapshot}
    \State $\text{snapshot\_ref} \gets \text{NULL}$
    \If{$\textsc{ShouldTakeSnapshot}(\ell, {id+1})$}
        \State $\text{snapshot\_ref} \gets \textsc{SealSnapshot}(C_{next})$
        \If{$\text{snapshot\_ref} = \text{NULL}$}
            \State \textbf{Abort} ``Snapshot Sealing Failed''
        \EndIf
    \EndIf

    \State \text{// 6. Append the new record to the local ledger}
    \State $\mathcal{L}_{\text{local}} \gets \textsc{UnsealLocalLedger}()$
    \State $\textsc{Append}(\mathcal{L}_{\text{local}}[\ell], \allowbreak ({id+1}, h_{next}, \sigma_{commit}, \allowbreak \text{snapshot\_ref}))$
    \State $\textsc{SealLocalLedger}(\mathcal{L}_{\text{local}})$
    \State $\textsc{DiscardPending}(\rho^{\mathrm{pend}}_{id+1})$

    \State \Return \textit{success}
\EndFunction
\end{algorithmic}
\end{algorithm}

\paragraph{Approved Out-of-Band Updates}
\label{sec:approved-update}
Some legitimate changes originate outside the query-driven workflow
(e.g., the user edits a \texttt{SKILL.md}). On its own, CG cannot distinguish
such a change from tampering, so it admits an out-of-band change only when
the trusted user approves it. An approved update $u = (\kappa, x, x')$
names a component $\kappa$ of the contextual state together with its exact
content before ($x$) and after ($x'$) the change. CG then admits the update
through the ordinary workflow: \cref{alg:approved} in~\cref{sec:deploy-details}
checks that the local component $\kappa$ holds exactly $x'$ and records
the approval, and the next query's cycle verifies the predecessor state
with $x$ in $\kappa$'s place and commits the new state with that query's
transition, where \textsc{PreviewApply} replaces $x$ with $x'$ and
\textsc{Finalize} has no further effect. The approval thus authorizes only
that one change. An approved update is recorded in the
ledger with the transition that carries it.

\paragraph{End-to-End Workflow}
We summarize the \agentsystem workflow in~\cref{alg:workflow}. The
workflow integrates remote attestation, contextual state verification,
planning, and state updates. In~\cref{alg:workflow},
\textsc{ResolveSessionState} resolves the session's ledger label $\ell$
and current sequence identifier $id$ from the tail of the sealed local
ledger. The remaining helper primitives are as follows:
\textsc{FirstRun} holds if and only if no ledger exists for the
deployment; \textsc{GenLedgerLabel} generates a fresh ledger label;
\textsc{GenNonce} samples an unpredictable nonce for each read
challenge, with a range large enough that a repeat is infeasible;
\textsc{PendingApprovals} returns the approved updates that~\cref{alg:approved}
holds in CG's memory inside the TEE and not yet carried by a committed transition,
and \textsc{ClearApprovals} marks them consumed once that transition
commits; \textsc{EncodeAction} canonically encodes each realized tool call
together with its result for the deterministic preview; \textsc{Finalize}
applies the previewed update to the local state files and reports
whether all local side effects are durable; and \textsc{ShouldTakeSnapshot}
implements the configurable snapshot policy. The workflow performs one
state update per successful query; an approved update (\cref{alg:approved})
is committed with the next query's transition. The workflow is fail-closed:
if any attestation, verification, or update step fails, the query aborts.

\begin{algorithm}[!t]
\caption{\agentsystem Workflow}
\label{alg:workflow}
\begin{algorithmic}[1]
\Function{RunQuery}{$q$}
    \State \text{// 1. Attest SCM and pin $PK_{SCM}$}
    \If{$\neg \textsc{Attested}()$}
        \State $PK_{SCM} \gets \textsc{AttestSCM}()$
    \EndIf

    \State \text{// 2. Fetch the current contextual state}
    \State $T \gets \textsc{MCP.ListTools}()$
    \State $(I, R, A) \gets \textsc{Memory.GetData}()$
    \State $M \gets \langle I, R, A \rangle$
    \State $C \gets \langle T, M \rangle$

    \State \text{// 3. Reconstruct the predecessor state}
    \State $U \gets \textsc{PendingApprovals}()$
    \State $C_{prev} \gets C$
    \ForAll{$u = (\kappa, x, x')$ in $U$}
        \If{$C[\kappa] \neq x'$}
            \State \textbf{Abort} ``Unapproved Local Change''
        \EndIf
        \State $C_{prev} \gets C_{prev}[\kappa \mapsto x]$
    \EndFor

    \State \text{// 4. Compute the predecessor state digest}
    \State $h \gets \mathcal{H}(C_{prev})$

    \State \text{// 5. Initialize or verify the contextual state}
    \If{$\textsc{FirstRun}()$}
        \State $\ell \gets \textsc{InitializeState}(C, \mathcal{H}(C))$
        \State $id \gets 0$
    \Else
        \State \text{// Resolve the current session}
        \State $(\ell, id) \gets \textsc{ResolveSessionState}(q)$
        \If{$\textsc{VerifyState}(\ell, id, h) \neq \textit{success}$}
            \State \Return \textit{failure}
        \EndIf
    \EndIf

    \State \text{// 6. Generate a plan $p$}
    \State $p \gets \textsc{LLMPlan}(q, C)$

    \State \text{// 7. Execute the plan and accumulate actions}
    \State $\mathcal{A} \gets [q, p]$
    \ForAll{tool call $c$ in $p$}
        \State $r \gets \textsc{MCP.CallTool}(c)$
        \State \text{// Encode the realized call-output pair for preview}
        \State $a \gets \textsc{EncodeAction}(c, r)$
        \State $\mathcal{A}.\textsc{Append}(a)$
    \EndFor

    \State \text{// 8. Commit and persist the next contextual state}
    \If{$\textsc{UpdateState}(\ell, id, U \Vert \mathcal{A}, C_{prev})\! \neq\! \textit{success}$}
        \State \Return \textit{failure}
    \EndIf
    \State $\textsc{ClearApprovals}(U)$

    \State \Return \textit{success}
\EndFunction
\end{algorithmic}
\end{algorithm}

\section{Security Analysis}
\label{sec:analysis}

Under standard cryptographic and system assumptions, this section
proves that \agentsystem achieves Contextual State Continuity and
Historical Traceability.

\subsection{Preliminaries and Assumptions}
\label{sec:analysis-assumptions}

Our proofs rely on the following assumptions:
\begin{itemize}
\item (A1) SCM Ledger Correctness.
SCM implements the linearizable, append-only ledger of~\cref{alg:scm} on
Nimble~\cite{angel2023nimble}, whose safety and liveness we do not re-prove.
For each ledger $\ell$, a \textsc{Commit} succeeds only at the identifier
one greater than the tail, so the identifiers $id$ are strictly monotonic and
gap-free, and SCM issues a receipt only for an appended entry. $SK_{SCM}$
is held only by attested SCM instances and sealed in secure storage when
they stop.

\item (A2) Signature Unforgeability.
SCM signatures generated with its private key $SK_{SCM}$ are 
unforgeable.

\item (A3) Hash Function Security.
The hash function $\mathcal{H}$ is collision-resistant. It is
computationally infeasible to find distinct inputs $X \neq X'$ with
$\mathcal{H}(X) = \mathcal{H}(X')$. Since $\mathcal{H}$ is applied to the
injective canonical encoding of a contextual state (\cref{sec:cstate-define}),
it is likewise infeasible to find two distinct contextual states with the
same digest.

\item (A4) Trusted CG Execution and Non-Bypassability.
The Context Guard (CG) executes faithfully within its TEE boundary,
and its verification and update logic cannot be bypassed. Specifically, 
the agent proceeds to planning and execution only after CG has 
successfully verified the contextual state. Furthermore, CG finalizes 
any local contextual state update only after obtaining a valid receipt
from SCM. CG persists its local state with authenticated encryption,
so sealed state cannot be modified without detection. The trusted computing
base consists of the TEEs and the software they run: the Agent Host and
the attested SCM replicas (see \cref{fig:elephantagent-overview}).

\item (A5) Attested and Integrity-Protected SCM Access.
For a fixed ledger $\ell$, CG authenticates the SCM endpoint via remote
attestation and pins $PK_{SCM}$ for proof and receipt verification. All
communication between CG and SCM is integrity-protected (e.g., via TLS).
Thus, CG receives only authenticated SCM proofs and receipts for state
verification and updates. Conversely, SCM authenticates CG: for each ledger,
SCM accepts \textsc{Commit} operations only from its bound CG. This binding
is realized by the ledger label: CG generates $\ell$ (128-bit) at random
inside its TEE, holds it in its sealed state, and reveals it to SCM only
over the attested channel, so only the bound CG can name $\ell$ in a
\textsc{Commit} (\cref{sec:deploy-details}).

\end{itemize}

\subsection{Contextual State Continuity}

Contextual State Continuity (\cref{def:csc}) requires that, at every 
planning-and-execution cycle, the agent acts only on the latest 
faithfully evolved contextual state.

\begin{lemma}[Contextual State Integrity and Freshness]
\label{lem:continuity}
Fix a ledger $\ell$ maintained by SCM. If \textsc{VerifyState}
(\cref{alg:verify}) returns \textit{success}, then the contextual state
CG accepts for release to the agent satisfies the following:

\begin{enumerate}
    \item \emph{Integrity.}
    The accepted contextual state is a faithfully evolved contextual 
    state for ledger $\ell$.

    \item \emph{Freshness.}
    The accepted contextual state is not a stale prior state, but the 
    latest state for ledger $\ell$.
\end{enumerate}

Moreover, any attempt to make CG proceed on an out-of-band tampered 
contextual state causes \textsc{VerifyState} to return \textit{failure} 
or abort, and the agent does not proceed to planning or execution on
such a state.
\end{lemma}

\Cref{sec:proofs} gives the proof. We extend this single-step guarantee
to the whole execution by comparing the real system with the ideal
reference system (\cref{sec:sproperties}).

\begin{corollary}[Agreement of Real and Ideal Contextual State Sequences]
\label{cor:csc}
Fix an approved initial contextual state $C_0$, and consider a real
execution of~\cref{alg:workflow,alg:approved} from $C_0$ in the
presence of an adversary admitted by the threat model
in~\cref{sec:threat}. Let $Q = (q_1, q_2, \ldots)$ be the sequence of
elements the execution finalizes, i.e., whose cycles complete the
state update after a successful verification, let $\Omega$ be the
resolution of environment nondeterminism the finalized cycles consume,
and let
\[
\mathcal{L}^{*} = \big((0,h_0), (1,h_1), (2,h_2), \ldots\big)
\]
be the reference ledger induced by the ideal reference system from
$C_0$ under $Q$ and $\Omega$. Then, whenever CG verifies a predecessor
state $\hat{C}$ under the sequence identifier $\hat{t}$ resolved
by~\cref{alg:workflow}, and the execution has finalized the first $t$
elements of $Q$ at the linearization point of that verification's read
of the ledger (\cref{alg:verify}), the pair agrees with the reference
ledger, namely
\[
(\hat{t}, \mathcal{H}(\hat{C})) = (t, h_t).
\]
\end{corollary}

The proof proceeds by induction on $t$ (\cref{sec:proofs}).

\cref{lem:continuity} and~\cref{cor:csc} establish the Contextual State
Continuity property of \agentsystem. A cycle of~\cref{alg:workflow} releases
a contextual state only after a successful verification, and what it releases
is the verified predecessor, which agrees with the reference ledger by
\cref{cor:csc}; the agent plans on it with the approved replacements applied,
as the reference system's transition does. A cycle that aborts finalizes no
element and contributes no entry, satisfying the fail-closed clause
of~\cref{def:csc}.

\subsection{Historical Traceability}

Historical Traceability supports post-hoc recovery.

\begin{lemma}[Authorized, Logged, and Ordered Effective Transitions]
\label{lem:trace}
Fix a ledger $\ell$ maintained by SCM. Consider an execution of
\cref{alg:workflow} in which \textsc{UpdateState} is invoked after
a successful \textsc{VerifyState} and returns \textit{success},
advancing the local sequence identifier from $id$ to $id+1$. 
Then the following statements hold:
\begin{enumerate}
    \item SCM has appended the entry $(id+1, h_{next})$ to ledger
    $\ell$, as attested by a valid SCM receipt $\sigma_{commit}$.

    \item The local ledger $\mathcal{L}_{\text{local}}$ contains a record 
    $(id+1, h_{next},\allowbreak \sigma_{commit},\allowbreak \text{snapshot\_ref})$
    for that entry.

    \item The locally recorded transitions for ledger $\ell$ preserve
    the order of the corresponding SCM-authorized transitions.
\end{enumerate}
\end{lemma}

\Cref{sec:proofs} gives the proof.

\begin{corollary}[History and Recovery Soundness]
\label{cor:audit}
Fix a ledger $\ell$ and an adversary admitted by the threat model
in~\cref{sec:threat}. Let $R$ be the set of local ledger records for
$\ell$ and their snapshots that remain retained; the adversary may
tamper with, roll back, or delete the rest. Then the following two
statements hold. (1) History soundness: every transition that
\agentsystem treats as effective for $\ell$ has an authorized record
at its position in the history, and every record released to the
auditor is one from $R$ that carries a receipt verifying under
$PK_{SCM}$ for its position in the history.
(2) Recovery soundness: if recovery adopts a state $C_j$ from a record
in $R$, then $C_j$ is a faithfully evolved contextual state for $\ell$.
\end{corollary}

\Cref{sec:proofs} gives the proof. In practice, the auditor (\cref{sec:entities})
applies an audit policy to the recorded history to determine which
recorded state is undesired, say the first such step $k$; recovery then
targets the largest snapshot-retained index $j < k$. Recovery also closes
the superseded ledger: the terminal entry committed before re-initialization
(\cref{sec:design}) remains the tail of $\ell$ by Assumptions~(A1) and~(A5),
and no contextual state matches its digest by Assumption~(A3), so verification
against $\ell$ fails for every state.
Together, \cref{lem:trace} and~\cref{cor:audit} establish Historical Traceability
as defined in~\cref{def:ht}.

\section{\agentsystem Evaluation}
\label{sec:eval}

We integrate \agentsystem into DeepSeek Harness (DSH)~\cite{deepseek2026dsh}
and implement \elephantdsh. We then evaluate its protection and overhead.

\subsection{\elephantdsh}
\label{sec:eval:impl}

\elephantdsh augments the harness DSH (the Agent Host
in~\cref{fig:elephantagent-overview}) with CG and SCM.
We implement CG in about 12K lines. CG manages the contextual state
$C_t = \langle T_t, M_t \rangle$. The tool state $T_t$ covers both
native tools and MCP tools. The memory includes instructions $M^I_t$
(i.e., \texttt{AGENTS.md}), the transcript $M^R_t$ (i.e., the model's
session history), and the managed artifacts $M^A_t$ (i.e., files enrolled
for protection). CG applies \cref{alg:verify} and \cref{alg:update} at
the granularity of DSH's queries, following \cref{alg:workflow}. When
a query starts, CG verifies the current state against the latest
SCM-authorized digest (\cref{sec:design}); when the query completes,
CG commits the query's contextual state update. \elephantdsh snapshots
at every commit, so every committed query is a recoverable checkpoint.
CG admits an initial contextual state approved by the trusted user at
initialization. When DSH resumes a session, CG first verifies the resumed
contextual state against the latest SCM-authorized state. CG persists
its local ledger, snapshots, and pending records under AES-GCM sealing
(\cref{sec:deploy-details}). A recovery command, run by the trusted user,
releases the recorded history to an auditor only after cryptographically
verifying each record, and re-anchors an audited checkpoint as the first
entry of a fresh ledger (\cref{sec:eval:recovery}).

Applying \cref{alg:verify,alg:update} to DSH requires mediating every path
by which content reaches the model or the session store; we identified these
paths by auditing DSH's source and list how CG mediates each (\cref{fig:harness}).
(i) Send what was verified. DSH's built-in instruction loader reads
\texttt{AGENTS.md} itself and reloads it whenever the agent's file tools touch
the workspace, so \elephantdsh replaces it with a CG-owned loader that reads
each instruction file once per step, checks the contents against the approved
state, and places those same contents in the request. (ii) Check after the
last await. DSH still awaits other work between a plugin's inspection of a
model request or tool call and its dispatch, so a tool definition or handler
changed in that gap would be dispatched unchecked. \elephantdsh therefore adds
a synchronous guard at DSH's final model and tool dispatch points, its only
change to DSH's runtime code (two files). In four controlled tests, such a
change was dispatched in all four without the guard and in none with it. 
(iii) Accept only at commit. DSH flushes session events before model
and tool dispatch and drains them in the background, independently of query
boundaries, so CG owns the persistence provider and holds these events out of
the accepted session until the query's transition commits. (iv) Authenticate
before any rewrite. When DSH resumes a session, it appends repair and
constructor events before any plugin setup runs, so CG authenticates the
stored session inside the persistence provider, before those events are
appended. (v) Re-verify on reconnect. A reconnecting MCP server re-registers
its tools; unchanged descriptors pass the next verification, while a changed
descriptor fails it.

\begin{figure}[t]
\centering
\includegraphics[width=\linewidth]{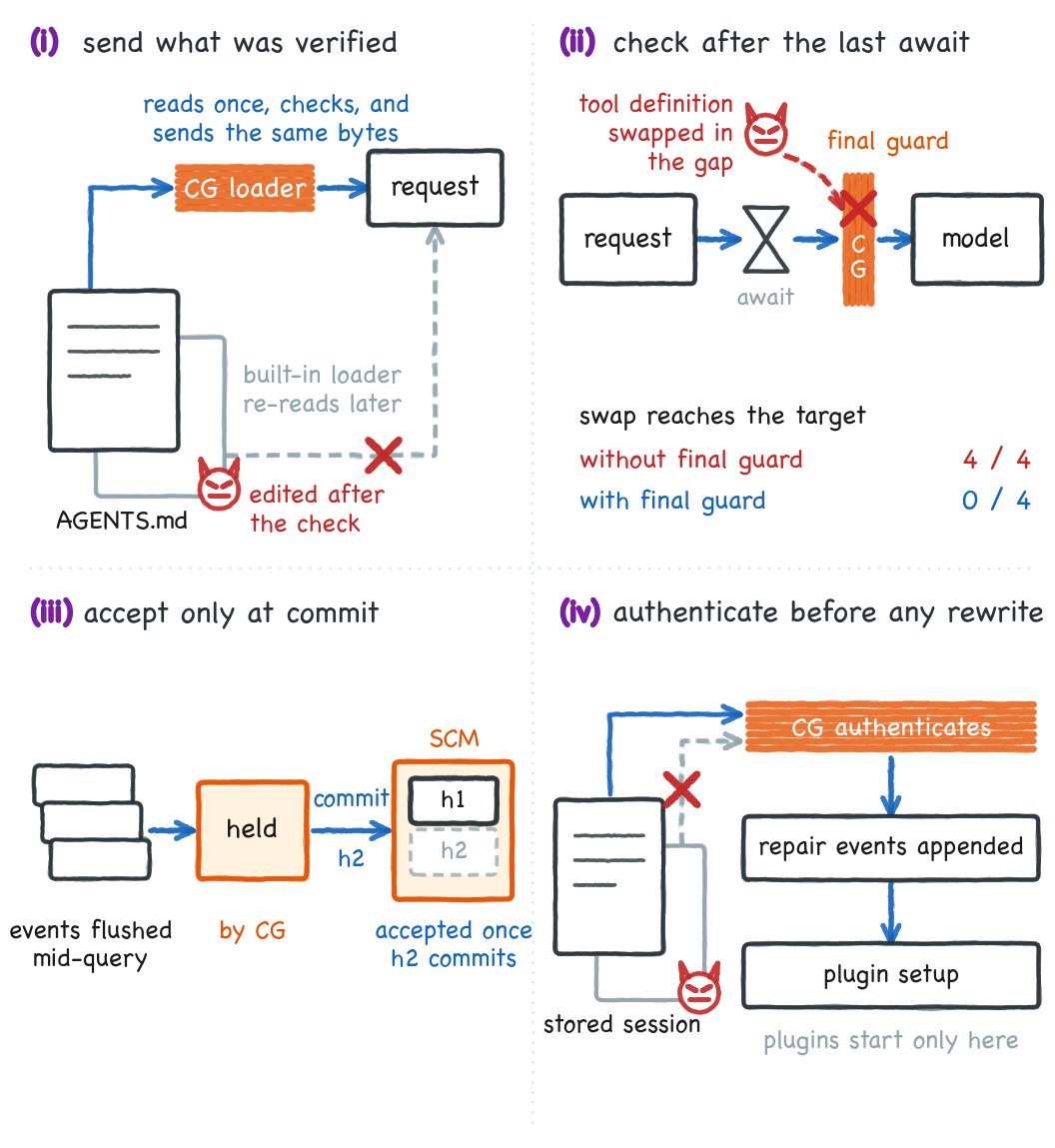}
\caption{
    How CG mediates the four in-process paths of \cref{sec:eval:impl},
    numbered as in the text.
}
\label{fig:harness}
\end{figure}

SCM is the Nimble-based replicated ledger service~\cite{angel2023nimble},
following the Rust implementation of Nimble. When CG starts, it attests
SCM as described in~\cref{sec:deploy-details}.

\subsection{Experimental Setup}
\label{sec:eval:setup}

Our experiments run on machines equipped with Intel Xeon 6982P-C CPUs that
support Intel TDX. All machines are ``ecs.g9i.xlarge'' instances rented from
Alibaba Cloud in the Beijing region and communicate over an intra-datacenter
LAN. Three instances run as TDX-protected VMs, which we denote as ``TDX1'',
``TDX2'', and ``TDX3''. A fourth instance, ``NT1'', runs without TDX
(i.e., TDX disabled).

We compare \elephantdsh against \nativedsh, the original DSH.
\nativedsh runs on NT1. \elephantdsh's harness and CG run on TDX1,
and SCM runs on TDX2 and TDX3, with three nodes as separate processes on
TDX2 and two on TDX3; the external MCP Servers of the security experiments
run on NT1. The overhead we report therefore includes the CG and SCM work;
we observe that the TDX VM itself adds negligible overhead to the coding
tasks we evaluate, consistent with Intel's characterization of TDX overhead
as concentrated in I/O-intensive workloads such as Redis and
MySQL~\cite{intel2025tdxperf}. The agent uses DeepSeek V4.1 Flash through
the official DeepSeek API. We run 10 fresh trials per experiment.

\subsection{Security Evaluation}
\label{sec:eval:security}

We conduct the two out-of-band attacks of~\cref{sec:threat}, the rug pull of
Scenario~I and the rollback of Scenario~II, against \elephantdsh, \nativedsh,
and other baseline agent defense systems, to evaluate their protection. Each
attack succeeds under the criterion defined in the corresponding scenario
in~\cref{sec:threat}. The in-band injection of Scenario~III is evaluated
separately in~\cref{sec:eval:recovery}.

\paragraph{Baselines}
The baselines are A-MemGuard~\cite{wei2025amemguard},
ACE~\cite{li2026ace}, MemLineage~\cite{ouyang2026memlineage}, Tool
Result Parsing~\cite{yu2026toolresult}, and MCP tool pinning~\cite{mcppin}.
A-MemGuard, ACE, Tool Result Parsing, and MCP tool pinning follow their
official repositories and configurations; MemLineage uses a third-party
implementation~\cite{memlineage-thirdparty}, as its official code is
unreleased. All baselines use DeepSeek V4.1 Flash and run the same
evaluation protocol as \elephantdsh, with 10 trials per~attack.

\paragraph{Results}
\cref{tab:agent-security-comparison} lists the results. 
No baseline defense stopped both attacks. Every one failed to prevent the
rollback, and all but ACE and tool pinning leaked under the rug pull.
In the three of ten rollback trials in which A-MemGuard did not leak,
its LLM consistency check happened to judge the stale policy unsafe for
the current task and returned no memory. ACE blocked the rug pull through
its information-flow analysis. MemLineage, which signs every memory entry,
leaked in all 10 rollback trials, as each entry of the stale state still
carries a valid signature. Likewise, tool pinning refused every rug pull,
but the rollback restored its pins together with the stale policy, so every
pin check passed. \elephantdsh rejected all out-of-band attack attempts:
under the rug pull, the tool descriptor no longer matched the approved one, and
under the rollback, the restored local contextual state differed from the latest
SCM-authorized state. In every rejected trial, no model request was dispatched,
no review service was contacted, and the SCM ledger was unchanged. The two
attacks illustrate the difference between empirical detection and a verifiable
continuity check: the rug pull's suspicious directive can alert the model or
the detection mechanisms, while the replayed state contained nearly nothing
for empirical detection to flag, and the model reliably acted on it.

\paragraph{Attacks on the Mechanism}
We also attack \agentsystem's own mechanism, instantiating the remaining
out-of-band attack vectors of~\cref{sec:attack-coverage} (V2, V4--V6, V8, and V9)
on \elephantdsh; \nativedsh has no CG or SCM to attack, so it is not included
in this set of experiments. After a benign session with one review call,
the adversary tampers with the memory files (V2), redirects CG's SCM connections
to an on-path proxy (V4) or a rogue SCM (V5), crashes CG after an SCM commit (V6),
tampers with or rolls back the local ledger and snapshots (V8), or restores the
pre-recovery state after a recovery (V9). All 140 attack trials (14 variants,
10 each; \cref{fig:defense}) failed: every attack on the state or the SCM channel
was rejected, every interrupted update completed on the next resume, and no attack
on the artifacts became effective.

\subsection{Performance Evaluation}
\label{sec:eval:overhead}

We use five coding tasks at three levels of difficulty (\cref{sec:eval-tasks}).
Each task has two phases. In the first phase, the agent reads the code, edits
it with the native tools, and runs the tests. In the second phase, a new
invocation resumes the session; in \elephantdsh, CG verifies the resumed
state against SCM before the agent continues. In the conversion task, the
resumed phase first applies an approved \texttt{AGENTS.md} update
(\cref{sec:approved-update}) that adds a docstring requirement, and the agent
must bring the code in line with the updated instructions; in the other tasks,
the agent re-runs the test suite. The tasks thus cover transcript growth,
session resume, and an approved instruction update. A hidden test suite
decides task success.

\paragraph{Results}
All 50 trials of each implementation succeeded. \cref{tab:overhead} reports
the latency of each task, and \cref{fig:overhead-breakdown} breaks the time
per trial down by component. Model time is close in both implementations;
the difference lies in non-model time, where \elephantdsh adds about 0.6\,s
per trial in the simple tasks to 1.0\,s in the hard task. The dominant and
growing share is CG's local work: encoding and hashing $C_t$, writing and
synchronizing the sealed records of each transition, and the per-invocation
session setup. \Cref{tab:ctx} shows how the local work latency and sealed
storage grow with $C_t$. This growth is from the implementation, not the
protocol: incremental constructions such as a hash chain over the append-only
transcript can make the per-query work proportional to the change in $C_t$.

\begin{table}[h]
\centering
\footnotesize
\caption{
Size of $C_t$ and its cost in \elephantdsh, medians over 10 trials.
Top: per-step time in a microbenchmark that grows $C_t$ under a
deterministic stand-in model. Bottom: coding tasks, with $C_t$ and
the corresponding total size of a trial's sealed snapshots.
}
\setlength{\tabcolsep}{3.5pt}
\begin{tabular}{lrrrrr}
\toprule
$C_t$ (KB) & 733 & 1316 & 1895 & 2473 & 2911 \\
\midrule
\nativedsh (ms) & 22 & 22 & 22 & 22 & 23 \\
\elephantdsh (ms) & 45 & 67 & 75 & 88 & 97 \\
\midrule
Task & WC & Med & Conv & Fix & QE \\
\midrule
$C_t$ (KB) & 61 & 74 & 71 & 191 & 479 \\
Snapshots (KB) & 161 & 195 & 175 & 506 & 1275 \\
\bottomrule
\end{tabular}
\label{tab:ctx}
\end{table}

\begin{table}[t]
\centering
\footnotesize
\caption{
Median latency (s) per trial over 10 trials on the coding tasks: word count
(WC), median (Med), and conversion (Conv) at the simple level, pricing fix
(Fix) at the medium, and query engine (QE) at the hard. Non-m.: non-model
time (total minus model request time). Overhead: \elephantdsh over \nativedsh
total time. Adj.: the same overhead with both implementations' model time
set to the \nativedsh median, removing model latency variation.
}
\setlength{\tabcolsep}{3.5pt}
\begin{tabular}{lrrrrrr}
\toprule
 & \multicolumn{2}{c}{\nativedsh} & \multicolumn{2}{c}{\elephantdsh} & & \\
\cmidrule(lr){2-3} \cmidrule(lr){4-5}
Task & Total & Non-m. & Total & Non-m. & Overhead & Adj. \\
\midrule
WC & 8.5 & 1.7 & 9.2 & 2.3 & 1.08 & 1.07 \\
Med & 10.1 & 1.7 & 10.5 & 2.3 & 1.04 & 1.06 \\
Conv & 9.7 & 1.7 & 10.5 & 2.3 & 1.08 & 1.06 \\
Fix & 17.1 & 1.8 & 17.7 & 2.6 & 1.04 & 1.05 \\
QE & 54.0 & 3.3 & 55.9 & 4.3 & 1.03 & 1.02 \\
\bottomrule
\end{tabular}
\label{tab:overhead}
\end{table}

In all, the relative overhead of \elephantdsh over \nativedsh stays between
1.03 and 1.08 (\cref{tab:overhead}). The model-time-adjusted overhead is
1.06--1.07 in the simple tasks and 1.05 and 1.02 in the medium and hard tasks.
Model time grows from 6.8 to 50.7\,s per trial, as requests
carry longer contexts and the model generates more code, while the added
non-model time grows only from 0.6 to 1.0\,s, so the relative overhead
declines slightly as tasks get harder.

We report the one-off SCM attestation cost separately: CG attests SCM at
first deployment and after recovery, taking a median of 1.61\,s over 10
trials.

\subsection{Recovery}
\label{sec:eval:recovery}

We exercise recovery after an in-band injection attack (Scenario~III) in 10
trials. Every step of the attack goes through the intended workflow, so
\elephantdsh commits every transition, 5 per session, including the one that
brings the injected notice into the contextual state. On \nativedsh, the
injected content persisted in all 10 trials.

The trusted user exports the recorded history with the recovery command, which,
for each record, verifies the stored SCM receipt and checks the unsealed snapshot
against the recorded digest before releasing its contents. The verified history
goes to an LLM auditor (GPT-6 Astra) that is instructed to examine it: it
reports whether retrieved content contains any directive that departs from the
intended instructions, and the first step at which such content entered. Of the
checkpoints before the flagged step, recovery targets the latest one that retains
a snapshot. The recovery command starts a new CG and reattests SCM. It then
re-verifies the target record and re-anchors the recovered state as the first
entry of a fresh SCM ledger. Recovery also terminates the ledger it supersedes,
so the pre-recovery polluted state stays unusable. \elephantdsh then resumes on
the recovered~state.

\begin{figure}[t]
\centering
\includegraphics[width=\linewidth]{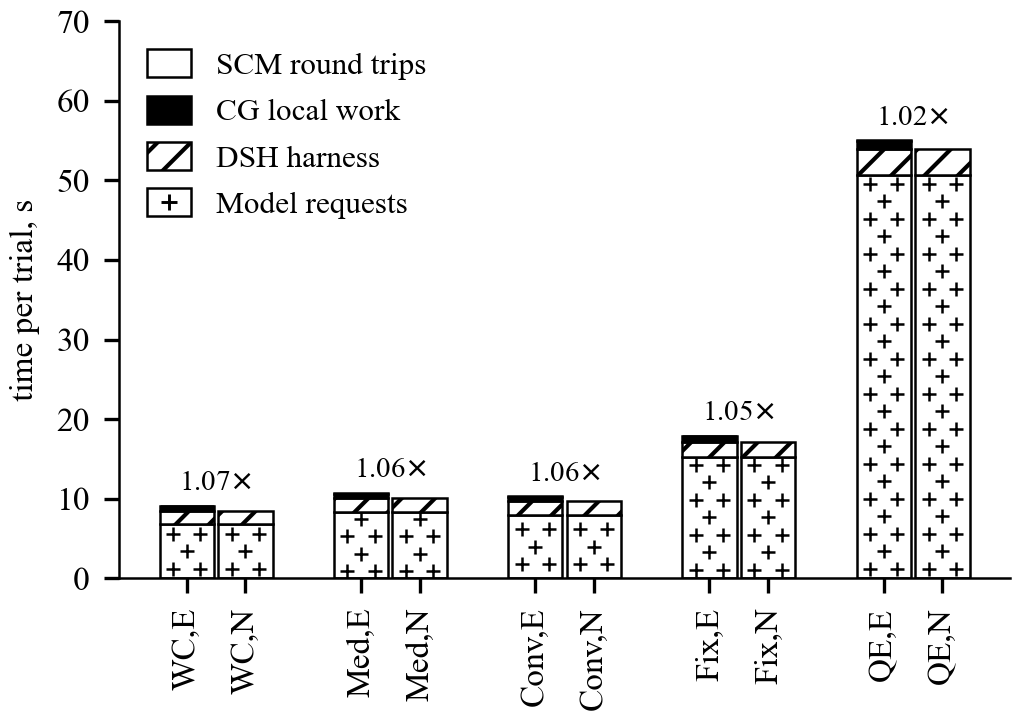}
\caption{
Median time per trial of \elephantdsh (E) and \nativedsh (N) on the five
coding tasks: word count (WC), median (Med), conversion (Conv), pricing
fix (Fix), and query engine (QE). Each bar is broken down into model requests,
the harness's own non-model work, CG's local work, and the SCM round trips.
The model time of both implementations is set to the \nativedsh median,
so each pair removes model latency variation; the number above each pair
is the adjusted overhead. The SCM round trips total about 60\,ms per trial,
barely visible at this scale.
}
\label{fig:overhead-breakdown}
\end{figure}

For the auditor, a control group of 10 uninjected sessions triggered no flag,
and in all 10 injected sessions it named the step where the notice entered,
step~3 of the 5 recorded. All 10 recoveries of \elephantdsh succeeded; every
recovery targeted the latest earlier checkpoint without injected content,
step~2. The recovery command takes a median of 1.80\,s, of which the
SCM attestation takes 1.61\,s. State verification and maintenance take
40\,ms, most of which unseals the snapshot and seals the fresh ledger's
first local ledger entry.

\paragraph{Tampering Outside $C_t$}
Workspace files lie outside $C_t$ (\cref{sec:cstate-define}), so a host
that rewrites one does not break continuity: the rewritten content reaches
the agent only when the agent needs to read the file, as a tool result,
which is the in-band path of Scenario~III. We test this path: the agent
consults a locally cached copy of the advisory, a workspace file outside
$C_t$, and the host rewrites that file with a notice between two queries.
In all 10 trials, continuity verification passed every query, and the notice
later entered $C_t$ as an ordinary tool result, again at step~3. Recovery
then proceeded as above: the LLM auditor named step~3 in all 10 trials,
and every recovery restored the step~2 checkpoint, free of the notice, and
terminated the superseded ledger.

\section{Limitations}
\label{sec:discuss}

\paragraph{Execution Under Interference and Crashes}
\agentsystem prioritizes safety over liveness, following the standard
posture of state-continuity systems~\cite{parno2011memoir,matetic2017rote,angel2023nimble,jin2024elephants}:
under adversarial interference, such as tampering with host-side persisted
state or blocking access to the latest SCM-authorized state, \agentsystem
fails closed. Likewise, authenticated encryption protects the sealed ledger
and snapshots against undetected modification but not against
deletion~\cite{angel2023nimble}. Deleting them stops the current session,
and a restarted session must re-approve the local content it starts from.
Continuity also does not imply exactly-once execution of external tool effects:
a crash between a tool's external effect and the commit of its transition can
duplicate the effect when the query is re-issued. This classic problem is
addressed by standard mechanisms at the tool boundary (e.g., idempotency
keys), which compose with \agentsystem: the transcript $M^R$ records each
committed invocation, an authenticated basis for reconciliation.

\paragraph{Recovery Scope and Freshness}
Recovery restores the scoped contextual state $C$, not the agent's full
runtime context; broader restoration is possible by incorporating more
components into $C$, at a corresponding cost in digest recomputation
and storage. Moreover, recovery verifies the integrity of the target state,
not its freshness: the restored state is guaranteed to be a faithfully
evolved contextual state, but the first undesired transition may have
occurred much earlier without being noticed, so recovering to a recent
known-good state may require periodic audits. Recovery is thus a deliberate
rollback. It is initiated by the trusted user based on the judgment of
a trusted auditor, and the audited recovery target is a checkpoint whose
receipt and digest verify, so a host adversary cannot use recovery to
reinstate a state of its own making. The deliberate rollback still discards
the authorized transitions after the target, which may include legitimate
updates; these can be re-applied by the user after recovery.

\section{Conclusion}
State continuity is a classical guarantee of trusted systems: a protected
computation must resume from exactly the state it last committed, on a
platform that can tamper with or roll back its persisted state at will.
\agentsystem establishes this guarantee for agentic systems. Doing so
requires deciding what to anchor and when: \agentsystem scopes the contextual
state to what the agent consumes as decision-making context, derives each
transition from the realized query, admits external changes only as approved
transitions, and verifies the content bound into each agent decision.
The leverage is that one check covers a family of attacks: tool-descriptor
rug pulls, memory tampering, and rollbacks are distinct out-of-band attacks,
yet each must ultimately present the agent with a contextual state that is
not the latest faithfully evolved one, and that is the single event
\agentsystem detects before any planning and execution takes place. The state
history that \agentsystem maintains further supports an auditable and verifiable
recovery, so semantic abuse that enters in band remains attributable to a
recorded transition and recoverable from an audited one afterwards. \agentsystem
is an add-on protocol, and our integration into the open-source DeepSeek Harness
shows the guarantee is practical: it rejected all rug-pull and rollback attempts,
whereas none of the existing defenses we compared against stopped both; it also
withstood all attacks on its own mechanism, recovered from a committed in-band
prompt injection in all trials, and completed every coding task at a modest
relative overhead of 1.02--1.07$\times$, with model latency variation removed.
Agentic systems today concentrate their defenses on checking what enters the
context, a first line that keeps its place against in-band abuse; we argue that
the context itself also deserves the formal continuity guarantee trusted systems
have long demanded of their state, and \agentsystem provides it.

\bibliographystyle{IEEEtran}
\bibliography{refs}

\appendices
\crefalias{section}{appendix}
\crefalias{subsection}{appendix}

\section{Attack Vector Coverage}
\label{sec:attack-coverage}

This section provides a catalog of the attack vectors
covered by \agentsystem's Contextual State Continuity (CSC) and Historical
Traceability (HT), under Assumptions (A1)--(A5) in \cref{sec:analysis}.

\paragraph{CSC Vectors}
All vectors below are covered by \agentsystem's CSC property.

\begin{itemize}
    \item \emph{V1: Tool descriptor mutation between queries
    (rug pull~\cite{invariantlabs2024mcp}).}
    An MCP Server registers a benign descriptor at onboarding and later
    mutates it between two queries to embed adversarial directives.
    The recomputed local contextual state digest changes while the
    SCM-authorized digest does not, so the digest equality check
    in~\cref{alg:verify} fails.

    \item \emph{V2: Memory file direct tampering.}
    An adversary with host-level access directly modifies a memory file outside
    the workflow without an approval. A modified sealed component fails the
    sealing integrity check when CG opens it; any other modification, such as
    an edited instruction file, causes \textsc{VerifyState}'s digest check to
    fail.

    \item \emph{V3: Memory rollback.}
    An adversary replays an older sealed version of the locally persisted
    memory (e.g., the interaction history). The locally recomputed digest
    corresponds to a prior step, whereas SCM returns the latest authorized
    digest; the mismatch triggers a verification failure.

    \item \emph{V4: Replayed or forged SCM messages.}
    An adversary on the path between CG and SCM injects an SCM proof issued
    for an earlier read challenge, or a proof or receipt fabricated without
    $SK_{SCM}$. Under Assumption~(A5), CG accepts SCM messages only over the
    integrity-protected channel to the attested SCM endpoint, so the message
    is rejected.

    \item \emph{V5: SCM endpoint impersonation.}
    An adversary runs a rogue service posing as SCM. Under Assumption
    (A5), CG authenticates the SCM endpoint via remote attestation and pins the
    attested $PK_{SCM}$; a rogue endpoint fails attestation or channel binding
    and is rejected before any of its messages is used.

    \item \emph{V6: Interrupted state update.}
    An adversary interrupts CG after the SCM commit of an update
    in~\cref{alg:update}. The pending update record, sealed before the SCM
    commit, enables recovery: CG compares it against the latest SCM state,
    finds that the interrupted transition has taken effect, and completes
    the update on the next resume. A partially finalized local state whose
    recomputed digest does not match the SCM-authorized digest is rejected
    by \textsc{VerifyState} and repaired from the pending record.
\end{itemize}

\paragraph{HT Vectors}
All vectors below are covered by \agentsystem's HT property.

\begin{itemize}
    \item \emph{V7: In-band semantic abuse (Scenario III).}
    An adversary induces the agent (e.g., through prompt injection in externally
    retrieved content) to execute harmful operations that are then committed as
    authorized transitions. A host that rewrites a workspace file outside $C_t$
    falls under this vector as well: the content may later reach planning as a
    tool result (\cref{sec:cstate-define}). The local ledger retains the ordered
    state transition history and the corresponding snapshots retain the contextual
    state contents; post-hoc recovery inspects the ledger and the relevant snapshots,
    and re-anchors a prior known-good state.

    \item \emph{V8: Local artifact tampering.}
    An adversary tampers with the local artifacts used for Historical Traceability,
    including the local ledger and its snapshots. We distinguish two cases. First,
    modifying these sealed artifacts is rejected by the sealing integrity check.
    Second, rolling back to older sealed artifacts only limits recovery to the
    earlier authorized versions.

    \item \emph{V9: Superseded session resumption.}
    An adversary retains a copy of the pre-recovery contextual state and
    restores it after a recovery. Verification of the restored state is
    directed at the superseded ledger, whose tail is the terminal entry
    committed during recovery (\cref{sec:design}); no contextual state matches
    its digest, so the restored state is rejected.
\end{itemize}

\section{Deployment Details of CG and SCM}
\label{sec:deploy-details}

\paragraph{SCM Interface and Approved Out-of-Band Update}
\cref{alg:scm} specifies the two SCM operations used in~\cref{sec:design}.
SCM maintains an append-only ledger $\mathcal{L}$ that maps each ledger label
to an ordered sequence of $(id, h)$ entries, and signs every receipt and proof
with $SK_{SCM}$. \cref{alg:approved} specifies how CG admits an approved update
$u = (\kappa, x, x')$ (\cref{sec:approved-update}): it records the approval,
and the next run of \cref{alg:workflow} verifies the predecessor state and
commits the update with that query's transition.
\cref{alg:approved} performs no update itself, by design. The approved change
reaches the environment first: at approval time, $\kappa$ already
holds $x'$, while the latest SCM-authorized state still binds $x$. The ledger
advances only from the state it last authorized, and
the workflow commits one transition per query (\cref{sec:design}), so
\cref{alg:workflow} reconstructs that predecessor in memory, with $x$ in
$\kappa$'s place, verifies it against SCM, and lets the next query's transition
carry the replacement: the approved change enters the ledger as every other
change does, within a verified, committed transition.

\begin{algorithm}[t]
\caption{SCM Interface}
\label{alg:scm}
\begin{algorithmic}[1]
\State \textbf{Internal State:} Ledger $\mathcal{L}$ (maps each ledger label $\ell$ to an ordered sequence of $(id, h)$)
\State \textbf{Secret:} SCM private signing key $SK_{SCM}$

\Statex \Comment{--- Linearizable Append Operation ---}
\Function{Commit}{$\ell, id, h$}
    \If{$id = 0$} \Comment{Ledger initialization}
        \If{$\ell \in \mathcal{L}$}
            \State \Return $(\text{False}, \bot)$ \Comment{Reject: ledger collision}
        \EndIf
        \State $\mathcal{L}[\ell] \gets [(0, h)]$
    \Else \Comment{State update}
        \If{$\ell \notin \mathcal{L}$}
            \State \Return $(\text{False}, \bot)$ \Comment{Reject: unknown ledger}
        \EndIf
        \State $(id_{last}, h_{last}) \gets \text{Tail}(\mathcal{L}[\ell])$
        \If{$id \neq id_{last} + 1$}
            \State \Return $(\text{False}, \bot)$ \Comment{Reject: rollback or fork}
        \EndIf
        \State $\text{Append}(\mathcal{L}[\ell], (id, h))$
    \EndIf

    \State $\sigma \gets \text{Sign}_{SK_{SCM}}(\ell, id, h)$ \Comment{Receipt}
    \State \Return $(\text{True}, \sigma)$
\EndFunction

\Statex \Comment{--- Linearizable Read Operation ---}
\Function{GetLatestState}{$\ell, n$}
    \If{$\ell \notin \mathcal{L}$}
        \State \Return $(\text{False}, \bot)$ \Comment{Reject: unknown ledger}
    \EndIf

    \State $(id_{last}, h_{last}) \gets \text{Tail}(\mathcal{L}[\ell])$
    \State $\sigma \gets \text{Sign}_{SK_{SCM}}(\ell, id_{last}, h_{last}, n)$
    \State $\pi \gets \langle \ell, id_{last}, h_{last}, n, \sigma \rangle$
    \State \Return $(\text{True}, \pi)$
\EndFunction

\end{algorithmic}
\end{algorithm}

\begin{algorithm}[t]
\caption{Approval of an Out-of-Band Update}
\label{alg:approved}
\begin{algorithmic}[1]
\Function{AdmitApproval}{$u = (\kappa, x, x')$}
    \State \text{// 1. Approve the update}
    \If{$\neg \textsc{ApprovePolicy}(u)$}
        \State \textbf{Abort} ``Update Not Approved''
    \EndIf

    \State \text{// 2. Check the local change against the approval}
    \State $C \gets \langle \textsc{MCP.ListTools}(), \textsc{Memory.GetData}() \rangle$
    \If{$C[\kappa] \neq x'$}
        \State \textbf{Abort} ``Unapproved Local Change''
    \EndIf

    \State \text{// 3. Record the approval for the next transition}
    \State $\textsc{RecordApproval}(u)$

    \State \Return \textit{success}
\EndFunction
\end{algorithmic}
\end{algorithm}

\paragraph{SCM Reconfiguration Under Failures}
If SCM loses some nodes, the failed nodes can be replaced through
Nimble's reconfiguration mechanism before quorum is lost.
However, if SCM loses a majority of its nodes, it preserves safety but
loses liveness, because no quorum remains to produce receipts or proofs.
In this case, \agentsystem fails closed until SCM is reconfigured.

\paragraph{SCM Endpoint Authentication}
CG authenticates the SCM service endpoint via TEE remote attestation
before any verification. The attestation follows a challenge--response
protocol. CG sends a fresh nonce $n_{CG}$ to SCM, and SCM returns a TDX
quote whose \texttt{report\_data} commits to
$\text{SHA256}(n_{CG}\,\|\,\texttt{context}\,\allowbreak\|\,PK_{SCM})$,
where $PK_{SCM}$ is the SCM service public key and binds the attested
SCM identity to the authenticated channel. CG verifies the TDX quote
with the Alibaba Cloud TDX quote-verification toolchain and checks the
SCM measurements against an allowlist. CG enforces channel binding by
checking that the TLS peer-certificate public key matches the attested
$PK_{SCM}$. CG then pins the attested $PK_{SCM}$ for subsequent verification.
CG caches the attestation record and re-attests when necessary (e.g., after
SCM reconfiguration). Conversely, SCM authenticates CG through the ledger
label: CG generates the 128-bit label at random inside its TEE, keeps it
in its sealed state, and reveals it to SCM only over the attested channel,
so a commit that names the ledger comes only from its bound CG. A forked
CG instance that holds the same sealed state can name the same ledger,
but SCM accepts at most one commit per sequence identifier, so the instance
that loses the race fails closed at its next verification.

\paragraph{Key Management}
In our experimental deployment, the secret keys of SCM and CG are both
loaded locally, as in the experimental deployments of prior state-continuity
systems~\cite{angel2023nimble,niu2022narrator,jin2024elephants,parno2011memoir}.
A production deployment of \agentsystem should provision both keys through
attestation-based key release (e.g., Azure Key Vault~\cite{microsoft2026skr}
or Alibaba Cloud KMS~\cite{alibabacloud2026kms}).

\section{Deferred Proofs}
\label{sec:proofs}

\subsection{Proof of~\cref{lem:continuity}}

\begin{proof}
By~\cref{alg:workflow}, before invoking \textsc{VerifyState},
\textsc{RunQuery} fetches the local contextual state and reconstructs
the predecessor state $C_{prev}$: it aborts unless every component named
by a pending approved update holds its approved replacement, and reverts
each such component to its approved predecessor content. It then computes
the digest $h = \mathcal{H}(C_{prev})$ and calls \textsc{VerifyState}$(\ell, id, h)$.
By~\cref{alg:verify}, \textsc{VerifyState} samples a fresh read-challenge nonce
$n_r$ and invokes $\textsc{SCM.GetLatestState}(\ell, n_r)$ to obtain a proof
$\pi = \langle \ell_{scm}, id_{scm}, h_{scm},\allowbreak n_{scm}, \sigma \rangle$.
It returns \textit{success} only if (1) the proof is authentic and bound
to the fresh nonce, namely, $n_{scm} = n_r$ and
$\textsc{VerifySig}(PK_{SCM},\allowbreak (\ell_{scm}, id_{scm}, h_{scm}, n_{scm}), \sigma) = \text{True}$;
and (2) the proof matches the observed local state summary, namely, 
$\ell_{scm} = \ell$, $id_{scm} = id$, and $h_{scm} = h$.

\paragraph{Integrity}
By condition (1) and Assumptions (A2) and (A5), the returned proof tuple 
$(\ell_{scm}, id_{scm}, h_{scm},\allowbreak n_{scm}, \sigma)$ accepted by CG
is authentic and untampered. By condition~(2), the locally accepted contextual
state has digest $h = h_{scm}$ and is bound to the same ledger $\ell$.
Let $C$ denote the SCM-authorized contextual state for that step, so
$h_{scm} = \mathcal{H}(C)$. By Assumption~(A5), every entry in ledger
$\ell$ was committed by the CG bound to $\ell$, and by Assumption~(A4),
CG commits only three kinds of digest: that of the approved initial
state (\cref{alg:init}); that of the successor state derived by
\textsc{PreviewApply} from the previously verified state, where a
transition realizes a query together with any approved updates it
carries (\cref{alg:update,alg:workflow,alg:approved}); or, during
recovery, that of the terminal marker that closes a superseded ledger (\cref{sec:design}).
A terminal marker is encoded disjointly from every contextual state, so
by Assumption~(A3) no contextual state matches its digest; since
verification precedes every commit (Assumption~(A4)), no entry follows
it, and a successfully verified tail --- and every entry below it ---
is a state digest. By induction over these entries --- the entry at
index $0$ anchors the approved initial state, and each subsequent entry
is derived from the state matching its predecessor --- the SCM-authorized state
$C$ is therefore a faithfully evolved contextual state for ledger $\ell$.
By condition~(2) and Assumption~(A3), the accepted local state is exactly
$C$: an out-of-band modified state $C' \neq C$ passing the check would
require $\mathcal{H}(C') = \mathcal{H}(C)$, so any such modification fails
condition~(2). This establishes the integrity of the accepted contextual state.

\paragraph{Freshness}
By Assumption~(A1), $\textsc{SCM.GetLatest}\allowbreak\textsc{State}(\ell, n_r)$
returns the tail of the ledger for $\ell$, namely the latest authorized state for 
that ledger. By condition~(1), the proof accepted by CG is authenticated
and bound to the fresh nonce $n_r$, so the accepted proof corresponds to
that latest authorized state for this read. By construction
of~\cref{alg:workflow}, the pair $(\ell, id)$ supplied to
\textsc{VerifyState} is resolved from the tail of the sealed local
ledger. If an adversary rolls back the locally observed sequence identifier
or any component of the local contextual state, then condition~(2) fails:
either $id_{scm} \neq id$ or $h_{scm} \neq h$. Thus, \textsc{VerifyState}
returns \textit{failure}, and the stale local state is rejected. A
rolled-back component that a pending approved update names is rejected
even earlier: it no longer holds its approved replacement, so
reconstruction aborts before \textsc{VerifyState} is invoked. Likewise,
making CG accept an outdated SCM proof would require replaying a proof
that does not match CG's fresh nonce $n_r$, or producing a new valid proof
for that nonce without SCM. The former is rejected by the nonce check:
nonces are unpredictable and never repeat across read challenges
(\cref{sec:design}). The latter contradicts Assumptions~(A2) and~(A5).
If SCM is unreachable, \textsc{VerifyState} aborts, resulting in fail-closed
behavior. Finally, by Assumption~(A4), the agent cannot bypass CG's
verification; thus any stale local contextual state is blocked before planning
or execution.
\end{proof}

\subsection{Proof of~\cref{cor:csc}}

\begin{proof}
We show by induction on $t$ that, once the execution has finalized the
first $t$ elements of $Q$, the latest authorized pair is $(t, h_t)$
and the latest finalized contextual state is $C_t$. The base case
holds because both systems start from the same approved initial contextual
state $C_0$, hence from the same initial pair $(0, h_0)$. For the inductive
step, assume the claim after $t$ finalized elements, and consider any
subsequent cycle of~\cref{alg:workflow}, running for a query $q$ that
carries the (possibly empty) set $U$ of pending approved updates.

The cycle proceeds only if, for each $u = (\kappa, x, x') \in U$, the
local component $\kappa$ holds $x'$, and \textsc{VerifyState} succeeds
on $C_{prev}$, the local state with each such $x$ back in place.
By \cref{lem:continuity} and the inductive hypothesis,
$\mathcal{H}(C_{prev}) = h_t$, so by Assumption~(A3) $C_{prev} = C_t$:
the verified predecessor is $C_t$ with pair $(t, h_t)$, which
establishes the corollary's claim; the cycle consumes it with the approved
replacements applied as authorized inputs, and for each $u \in U$ the
ideal precondition holds, since component $\kappa$ of $C_t$ is $x$. An
approval whose predecessor content $x$ does not match the latest
authorized state, such as an earlier approval replayed after a later
update, yields $\mathcal{H}(C_{prev}) \neq h_t$ and fails verification.

If the cycle finalizes its element, $q_{t+1} = q$, then under the same
query and the same resolution $\Omega$, which fixes both the LLM
outputs and the tool invocation results the cycle consumes,
\textsc{PreviewApply} applies the approved replacements in $U$ and the
realized actions to $C_t$ --- exactly the transition of the reference
system --- so both systems induce the same next digest $h_{t+1}$, and
since \textsc{UpdateState} succeeded, the real system advances to the
pair $(t+1, h_{t+1})$ with contextual state $C_{t+1}$, exactly the
pair recorded in the reference ledger at step $t+1$. A cycle that
aborts instead finalizes no element and leaves the latest authorized
pair at $(t, h_t)$: by commit-before-finalize (\cref{alg:update}),
either SCM did not advance and the pending transition is discarded, or
SCM advanced to the pending digest and the pending-record resolution
(\cref{sec:design}) completes the update, which is the finalizing case
above. The claim therefore follows by induction.
\end{proof}

\subsection{Proof of~\cref{lem:trace}}

\begin{proof}
By~\cref{alg:workflow}, CG invokes \textsc{VerifyState} before
calling \textsc{UpdateState}. By~\cref{lem:continuity}, this ensures that the
contextual state from which the transition is derived (the local state,
with any pending approved component reverted to its predecessor content)
is consistent with the latest SCM-authorized state for ledger $\ell$.

Consider an execution in which \textsc{UpdateState} returns \textit{success}. 
On this \textit{success} path, $\textsc{SCM.Commit}(\ell, id+1, h_{next})$ 
must have returned $(\textit{status}=\text{True}, \sigma_{commit})$, and CG 
must have verified $\sigma_{commit}$ under $PK_{SCM}$ before finalizing any 
local side effects. By Assumptions~(A1) and~(A2), $\sigma_{commit}$ is 
therefore a valid SCM receipt attesting that SCM appended the entry
$(id+1, h_{next})$ to ledger $\ell$, establishing statement~(1).

Moreover, \textsc{UpdateState} returns \textit{success} only after CG 
appends the record $(id+1, h_{next}, \sigma_{commit}, \text{snapshot\_ref})$
to the local ledger $\mathcal{L}_{\text{local}}[\ell]$ and seals it
(\cref{alg:update}), where $\text{snapshot\_ref}$ is either a reference
to a snapshot created by the snapshot policy or $\text{NULL}$. This 
establishes statement~(2).

Finally, by Assumption~(A1), SCM accepts commits only at consecutive
identifiers for each ledger $\ell$, so SCM-authorized transitions for
ledger $\ell$ are ordered by their identifiers. Because each successful
local record carries the same incremented identifier $id+1$ as the
corresponding authorized transition, the sequence of locally recorded
transitions preserves the order of the corresponding SCM-authorized
transitions. Assumption~(A4) ensures that CG cannot be bypassed on
the update path, so each locally recorded transition follows this
commit-before-finalize workflow. This establishes statement~(3).
\end{proof}

\subsection{Proof of~\cref{cor:audit}}

\begin{proof}
Statement~(1) is a statement about the adversary's power over the local
artifacts, and we argue it case by case over what the host adversary can
do to a record.
\begin{enumerate}
    \item \emph{Modification.} By Assumption~(A4), the local ledger and
    snapshots are persisted with authenticated encryption under a key
    confined to CG's TEE. A record the adversary modifies out of band
    fails the sealing integrity check when CG opens it, so it is not
    treated as effective and is not released to the auditor. A sealed
    record the adversary replaces with another validly sealed one fails
    the receipt or digest check at its position in the history and is
    likewise rejected (V2, V8 of~\cref{sec:attack-coverage}).
    \item \emph{Fabrication.} To be treated as effective or released, a
    record must carry a receipt $\sigma$ that verifies under $PK_{SCM}$
    for the tuple $(\ell, id, h)$. By Assumption~(A2), the adversary
    cannot forge such a receipt without $SK_{SCM}$, and by
    Assumption~(A1) SCM issues a receipt only for an entry it has
    appended, at the identifier one greater than the ledger tail. Hence
    every record that passes the check is one SCM actually authorized at
    its position in the history.
    \item \emph{Rollback.} By Assumption~(A1) the ledger for $\ell$ is
    append-only and the recorded sequence identifiers are strictly
    monotonic and gap-free. Rolling the artifacts back to an earlier
    version therefore yields a prefix of the authorized history for
    $\ell$, not a different one. Every retained record still carries its
    original receipt and is still ordered as SCM authorized it, so the
    adversary can remove records from the retained set --- shrinking $R$
    and costing availability --- but cannot make an unauthorized
    transition effective or put the retained records out of order.
    \item \emph{Deletion.} Deleting records likewise only shrinks $R$:
    what remains is a subset of the authorized history in which, by the
    same monotonicity, each retained record keeps its receipt and its
    recorded position, and CG cannot open a deleted record, so it never
    treats a deleted transition as effective.
\end{enumerate}
These four cases exhaust the adversary's effect on the local artifacts.
In every case the negative part of statement~(1) holds --- no
unauthorized transition becomes effective, and no record is released
without a verifying receipt --- while availability may be lost, because
$R$ may be smaller than the full history. This matches the fail-closed
posture of \cref{def:csc}: the adversary can degrade the guarantee's
scope, not its soundness.

For statement~(2), suppose recovery adopts $C_j$ from a record in $R$.
The recovery check of \cref{sec:design} passes for that record only if
its retained receipt $\sigma_j$ verifies for the recorded digest $h_j$
under $PK_{SCM}$ and the snapshot contents unsealed for it
recompute to $h_j$. By Assumption~(A2), a verifying $\sigma_j$ is an
SCM-issued receipt for $(\ell, j, h_j)$; by statement~(1), the record is
at its authorized position; and by Assumption~(A3), the contents that
recompute to $h_j$ are exactly the state $C_j$ whose digest is
$h_j$; in particular, the record is not the terminal marker of a
superseded ledger, which is encoded disjointly from every contextual
state and has no snapshot contents recomputing to its digest. Since
every non-terminal entry SCM authorizes is the digest of a faithfully
evolved contextual state (the induction in~\cref{lem:continuity}),
$C_j$ is a faithfully evolved contextual state for $\ell$, and
\cref{alg:init} re-anchors it as the first entry of a fresh ledger.
\end{proof}

\section{Coding Tasks}
\label{sec:eval-tasks}

Of the five coding tasks in \cref{sec:eval:overhead}, the three simple ones
work on small Python repositories: adding a median function,
fixing a whitespace bug in a word-count function, and adding a
temperature-conversion function. The medium task works on a pricing library
of about 250 lines in seven modules: the agent receives a bug report about a
wrong checkout total, locates two causes in different modules against the
README's pricing rules, fixes them, and adds a regression test.
The hard task works on a query engine running a small SQL dialect over CSV
files, about 840 lines in twelve modules with a 31-test suite: adding
\texttt{ORDER BY} and \texttt{LIMIT} requires coordinated changes across the
lexer, syntax tree, parser, evaluator, documentation, and tests.

\end{document}